\documentclass[11pt]{article}
\usepackage{fullpage}

\date{}

\title{\bfseries Best-item Learning in Random Utility Models with Subset Choices} 

\author{
Aadirupa Saha\thanks{Indian Institute of Science, Bangalore, India. {\tt aadirupa@iisc.ac.in}}, \and Aditya Gopalan \thanks{Indian Institute of Science, Bangalore, India. {\tt aditya@iisc.ac.in} }
}

\usepackage{microtype}
\usepackage{cancel}
\usepackage{graphicx}
\usepackage{booktabs} 
\usepackage{amssymb}
\usepackage{color}
\usepackage{amsmath}
\usepackage{amsthm}
\usepackage{thmtools}
\usepackage{thm-restate}
\usepackage{algorithm}
\usepackage{algorithmic}

\usepackage{natbib}
\usepackage{hyperref}
\usepackage{nameref}

\allowdisplaybreaks

\newtheorem{thm}{Theorem}
\newtheorem{lem}[thm]{Lemma}

\newtheorem{cor}[thm]{Corollary}
\newtheorem{defn}[thm]{Definition}

\newtheorem{rem}{Remark}

\newcommand{\R}{{\mathbb R}}

\newcommand{\E}{{\mathbf E}}

\newcommand{\1}{{\mathbf 1}}

\newcommand{\cA}{{\mathcal A}}

\newcommand{\cE}{{\mathcal E}}

\newcommand{\cF}{{\mathcal F}}
\newcommand{\cG}{{\mathcal G}}

\newcommand{\cD}{{\mathcal D}}

\newcommand{\cR}{{\mathcal R}}

\newcommand{\X}{{\mathcal X}}

\newcommand{\hp}{{\hat p}}
\newcommand{\p}{{\mathbf p}}

\newcommand{\sm}{\setminus}

\newcommand{\btheta}{\boldsymbol \theta}

\newcommand{\bSigma}{\boldsymbol \Sigma}

\newcommand{\bnu}{{\boldsymbol \nu}}

\newcommand{\bsigma}{\boldsymbol \sigma}

\newcommand{\smi}[1]{{S \setminus #1}}

\def \tf{{ top-$m$ ranking feedback}}
\def \iidn{{\it `independent and identically distributed noise'}}
\def \nn{{\it Gaussian}}
\def \xn{{\it Exponential}}
\def \gn{{\it Gumbel}}

\def \un{{\it Uniform}}

\def \ratio{{\it Advantage-Ratio}}
\def \mrat{{\it Min-AR}}

\def \algpp{{\it Sequential-Pairwise-Battle}}
\def \algp{{Seq-PB}}
\def \algpm{{mSeq-PB}}

\def \rb{{\it Rank-Breaking}}

\newcommand{\red}[1]{\textcolor{red}{#1}}

\begin{document}

%

%
\maketitle

\begin{abstract}
	We consider the problem of PAC learning the most valuable item from a pool of $n$ items using sequential, adaptively chosen plays of subsets of $k$ items, when, upon playing a subset, the learner receives relative feedback sampled according to a general Random Utility Model (RUM) with independent noise perturbations to the latent item utilities. We identify a new property of such a RUM, termed the minimum advantage, that helps in characterizing the complexity of separating pairs of items based on their relative win/loss empirical counts, and can be bounded as a function of the noise distribution alone. We give a learning algorithm for general RUMs, based on pairwise relative counts of items and hierarchical elimination, along with a new PAC sample complexity guarantee of $O(\frac{n}{c^2\epsilon^2} \log \frac{k}{\delta})$ rounds to identify an $\epsilon$-optimal item with confidence $1-\delta$, when the worst case pairwise advantage in the RUM has sensitivity at least $c$ to the parameter gaps of items. Fundamental lower bounds on PAC sample complexity show that this is near-optimal in terms of its dependence on $n,k$ and $c$. 
\end{abstract}


\section{Introduction}
\label{sec:intro}

Random utility models (RUMs) are a popular and well-established framework for studying behavioral choices by individuals and groups \cite{thurstone1927law}. In a RUM with finite alternatives or items, a distribution on the preferred alternative(s) is assumed to arise from a random utility drawn from a distribution for each item, followed by rank ordering the items according to their utilities. 

Perhaps the most widely known RUM is the Plackett-Luce or multinomial logit model \cite{plackett1975analysis,luce2012individual} which results when each item's utility is sampled from an additive model with a Gumbel-distributed perturbation. It is unique in the sense of enjoying the  property of independence of irrelevant attributes (IIA), which is often key in permitting efficient inference of Plackett-Luce models from data \cite{KhetanOh16}. Other well-known RUMs include the probit model \cite{bliss1934method} featuring random Gaussian perturbations to the intrinsic utilities, mixed logit, nested logit, etc.

A long line of work in statistics and machine learning focuses on estimating RUM properties from observed data \cite{AzariRB+14,zhao2018learning,soufiani2013generalized}. Online learning or adaptive testing, on the other hand, has shown efficient ways of identifying the most attractive (i.e., highest utility) items in RUMs by learning from relative feedback from item pairs or more generally subsets \cite{Busa_pl,SGwin18,SueIcml+17}. However, almost all existing work in this vein exclusively employs the Plackett-Luce model, arguably due to its very useful IIA property, and our understanding of learning performance in other, more general RUMs has been lacking. We take a step in this direction by framing the problem of sequentially learning the best item/items in general RUMs by adaptive testing of item subsets and observing relative RUM feedback. In the process, we uncover new structural properties in RUMs, including models with exponential, uniform, Gaussian (probit) utility distributions, and give algorithmic principles to exploit this structure, that permit provably sample-efficient online learning and allow us to go beyond Plackett-Luce.

{\bf Our contributions:} We introduce a new property of a RUM, called the (pairwise) {\em advantage ratio}, which essentially measures the worst-case relative probabilities between an item pair across all possible contexts (subsets) where they occur. We show that this ratio can be controlled (bounded below) as an affine function of the relative strengths of item pairs for RUMs based on several common centered utility distributions, e.g., exponential, Gumbel, uniform, Gamma, Weibull, normal, etc., even when the resulting RUM does not possess analytically favorable properties such as IIA.

We give an algorithm for sequentially and adaptively PAC (probably approximately correct) learning the best item from among a finite pool when, in each decision round, a subset of fixed size can be tested and top-$m$ rank ordered feedback from the RUM can be observed. The algorithm is based on the idea of maintaining pairwise win/loss counts among items, hierarchically testing subsets and propagating the surviving winners -- principles that have been shown to work optimally in the more structured Plackett-Luce RUM \cite{Busa_pl,SGwin18}.

In terms of performance guarantees, we derive a PAC sample complexity bound for our algorithm: when working with a pool of $n$ items in total with subsets of size-$k$ chosen in each decision round, the algorithm terminates in $O(\frac{n}{c^2\epsilon^2} \log \frac{k}{\delta})$ rounds where $c$ is a lower bound on the advantage ratio's sensitivity to intrinsic item utilities. This can in turn be shown to be a property of only the RUM's perturbation distribution, independent of the subset size $k$. A novel feature of the guarantee is that, unlike existing sample complexity results for sequential testing in the Plackett-Luce model, it does not rely on specific properties like IIA which are not present in general RUMs. We also extend the result to cover top-$m$ rank ordered feedback, of which winner feedback ($m = 1$) is a special case. Finally, we show that the sample complexity of our algorithm  is order-wise optimal across RUMs having a given advantage ratio sensitivity $c$, by arguing an information-theoretic lower bound on the sample complexity of any online learning algorithm.
	
Our results and techniques represent a conceptual advance in the problem of online learning in general RUMs, moving beyond the Plackett-Luce model for the first time to the best of our knowledge.

\textbf{Related Work:}
For classical multiarmed bandits setting, there is a well studied literature on PAC-arm identification problem \cite{Even+06,Audibert+10,Kalyanakrishnan+12,Karnin+13,LilUCB}, where the learner gets to see a noisy draw of absolute reward feedback of an arm upon playing a single arm per round. On the contrary, learning to identify the best item(s) with only relative preference information (ordinal as opposed to cardinal feedback) has seen steady progress since the introduction of the dueling bandit framework \cite{Zoghi+13} with pairs of items (size-$2$ subsets) that can be played, and subsequent work on generalisation to broader models both in terms of distributional parameters \cite{Yue+09, Adv_DB,Ailon+14, Zoghi+15MRUCB} as well as combinatorial subset-wise plays \cite{MohajerIcml+17,pbo,SG18,Sui+17}. 
There have been several developments on the PAC objective for different pairwise preference models, such as those satisfying stochastic triangle inequalities and strong stochastic transitivity \citep{BTM}, general utility-based preference models \citep{SAVAGE}, the Plackett-Luce model \citep{Busa_pl} and the Mallows model \citep{Busa_mallows}]. Recent work has studied PAC-learning objectives other than identifying the single (near) best arm, e.g. recovering a few of the top arms \citep{Busa_top,MohajerIcml+17}, or the true ranking of the items \citep{Busa_aaai,falahatgar_nips}. Some of the recent works also extended the PAC-learning objective with relative subsetwise preferences \cite{SGrank18,ChenSoda+17,ChenSoda+18,SGwin18,Ren+18}.

However, none of the existing work considers strategies to learn efficiently in general RUMs with subset-wise preferences and to the best of our knowledge we are the first to address this general problem setup. 
In a different direction, there has been work on batch (non-adaptive) estimation in general RUMs, e.g., \cite{zhao2018learning,soufiani2013generalized}; however, this does not consider the price of active learning and the associated exploration effort required as we study here. A related body of literature lies in dynamic assortment selection, where the goal is to offer a subset of items to customers in order to maximise expected revenue, which has been studied under different choice models, e.g. Multinomial-Logit \citep{assort-mnl}, Mallows and mixture of Mallows \citep{assort-mallows}, Markov chain-based choice models \citep{assort-markov}, single transition model \citep{assort-stm} etc., but again each of this work addresses a given and a very specific kind of choice model, and their objective is more suited to regret minimization type framework where playing every item comes with a associated cost.

\textbf{Organization:} We give the necessary preliminaries and our general RUM based problem setup in Section \ref{sec:prelims}. The formal description of our feedback models and the details of $(\epsilon,\delta)$-best arm identification problem is given in Section \ref{sec:prob}. In Section \ref{sec:pair_pref}, we analyse the pairwise preferences of item pairs for our general RUM based subset choice model and introduce the notion of \ratio\, connecting subsetwise scores to pairwise preferences. Our proposed algorithm along with its performance guarantee and also matching lower bound analysis is given in Section \ref{sec:wi}. We further extend the above results to a more general top-$m$ ranking feedback model in Section \ref{sec:tr}.
Section \ref{sec:conclusion} finally conclude our work with certain future directions. All the proofs of results are moved to the appendix.

\newcommand{\rumk}{RUM$(k,\btheta)$}

\section{Preliminaries}
\label{sec:prelims}
{\bf Notation.} We denote by $[n]$ the set $\{1,2,...,n\}$. For any subset $S \subseteq [n]$, let $|S|$ denote the cardinality of $S$. 
When there is no confusion about the context, we often represent (an unordered) subset $S$ as a vector, or ordered subset, $S$ of size $|S|$ (according to, say, a fixed global ordering of all the items $[n]$). In this case, $S(i)$ denotes the item (member) at the $i$th position in subset $S$.   
$\bSigma_S = \{\sigma \mid \sigma$  is a permutation over items of $ S\}$, where for any permutation $\sigma \in \Sigma_{S}$, $\sigma(i)$ denotes the element at the $i$-{th} position in $\sigma, i \in [|S|]$.
$\1(\varphi)$ is generically used to denote an indicator variable that takes the value $1$ if the predicate $\varphi$ is true, and $0$ otherwise. 
$x \vee y$ denotes the maximum of $x$ and $y$, and $Pr(A)$ is used to denote the probability of event $A$, in a probability space that is clear from the context.

\subsection{Random Utility-based Discrete Choice Models}
\label{sec:RUM}  
A discrete choice model specifies the relative preferences of two or more discrete alternatives in a given set. %
%
{Random Utility Models} (RUMs) are a widely-studied class of discrete choice models; they assume a (non-random) ground-truth utility score $\theta_{i} \in \R$ for each alternative $i \in [n]$, and assign a distribution $\cD_i(\cdot|\theta_{i})$ for scoring item $i$, where $\E[\cD_i \mid \theta_i] = \theta_i$. To model a winning alternative given any set $S \subseteq [n]$, one first draws a random utility score $X_{i} \sim \cD_i(\cdot|\theta_{i})$ for each alternative in $S$, and selects an item with the highest random score. More formally, the probability that an item $i \in S$ emerges as the {\em winner} in set $S$ is given by:
\vspace{-1pt}
\begin{align}
\label{eq:prob_rum}
Pr(i|S) = Pr(X_i > X_j ~~\forall j \in S\sm \{i\} )
\end{align}

In this paper, we assume that for each item $i \in [n]$, its random {\em utility score} $X_i$ is of the form $X_i = \theta_i + \zeta_i$, where all the $\zeta_i \sim \cD$ are `noise' random variables drawn independently from a probability distribution $\cD$. 

A widely used RUM is the {\it Multinomial-Logit (MNL)} or {\it Plackett-Luce model (PL)}, where the $\cD_i$s are taken to be independent Gumbel$(0,1)$ distributions with location parameters $0$ and scale parameter $1$ \citep{Az+12}, which results in score distributions $Pr(X_i \in  [x,x + dx]) = e^{-(x - \theta_i)}e^{-e^{-(x - \theta_i)}} dx$, $\forall i \in [n]$. Moreover, it can be shown that the probability that an alternative $i$ emerges as the winner in any set $S \ni i$ is simply proportional to its score parameter:
$
Pr(i|S) = \frac{e^{\theta_i}}{\sum_{j \in S}e^{\theta_j}}.
$

Other families of discrete choice models can be obtained by imposing different probability distributions over the iid noise $\zeta_i \sim \cD$; e.g., 

\begin{enumerate}
\item \xn\, noise: $\cD$ is the Exponential$(\lambda)$ distribution ($\lambda > 0$).
\item Noise from \textit{Extreme value distributions}: $\cD$ is the Extreme-value-distribution$(\mu,\sigma,\xi)$ ($\mu \in \R, \sigma > 0, \xi \in \R$). Many well-known distributions fall in this class, e.g., \textit{Frechet, Weibull, Gumbel}. For instance, when $\chi = 0$, this reduces to the \gn$(\mu,\sigma)$ distribution.
\item \un\, noise: $\cD$ is the (continuous) Uniform$(a,b)$ distribution ($a,b \in \R, b > a$).
\item \nn\, or Frechet, Weibull, Gumbel noise: $\cD$ is the Gaussian$(\mu,\sigma)$ distribution ($\mu \in \R, \sigma > 0$).
\item \textit{Gamma} noise: $\cD$ is the Gamma$(k,\xi)$ distribution (where $k,\xi > 0$).
\end{enumerate}
Other distributions $\cD$ can alternatively be used for modelling the noise distribution , depending on desired tail properties, domain-specific information, etc.

Finally, we denote a RUM choice model, comprised of an instance $\btheta = (\theta_1,\theta_2,\ldots, \theta_n)$ (with its implicit dependence on the noise distribution $\cD$)  along with a playable subset size $k \leq n$, by \rumk.


\section{Problem Setting}
\label{sec:prob}
We consider the probably approximately correct (PAC) version of the sequential decision-making problem of finding the {best item} in a set of $n$ items, by making only subset-wise comparisons. %

Formally, the learner is given a finite set $[n]$ of $n > 2$ items or `arms'\footnote{terminology borrowed from multi-armed bandits} along with a playable subset size $k \leq n$. At each decision round $t = 1, 2, \ldots$, the learner selects a subset $S_t \subseteq [n]$ of $k$ 
distinct items, and receives (stochastic) feedback depending on (a) the chosen subset $S_t$, and (b) a \rumk\, choice model with parameters $\btheta = (\theta_1,\theta_2,\ldots, \theta_n)$ a priori unknown to the learner. The nature of the feedback can be of several types as described in Section \ref{sec:feed_mod}. For the purposes of analysis, we assume, without loss of generality\footnote{under the assumption that the learner's decision rule does not contain any bias towards a specific item index}, that $\theta_1 > \theta_i \, \forall i \in [n]\setminus\{1\}$ for ease of  exposition\footnote{The extension to the case where several items have the same highest parameter value is easily accomplished.}. We define a \emph{best item} to be one with the highest score parameter: $i^* \in \underset{i \in [n]}{\text{argmax}}~\theta_i = \{1\}$, under the assumptions above.

\begin{rem}
Under the assumptions above, it follows that item $1$ is the \emph{Condorcet Winner} \cite{Zoghi+14RUCB} for the underlying pairwise preference model induced by \rumk.
\end{rem}



\subsection{Feedback models}
\label{sec:feed_mod}
We mean by `feedback model' the information received (from the `environment') once the learner plays a subset $S \subseteq [n]$ of $k$ items. Similar to different types of feedback models introduced  earlier in the context of the specific Plackett-Luce RUM \cite{SGwin18}, we consider the following feedback mechanisms:

\begin{itemize}
	
	\item \textbf{Winner of the selected subset (WI:} 
	The environment returns a single item $I \in S$, drawn independently from the probability distribution
	$
	\label{eq:prob_win}
	Pr(I = i|S) = Pr(X_i > X_j, ~\forall j \in S\sm \{i\} ) ~~\forall i \in S, \, S \subseteq [n].
	$
	
	\item \textbf{Full ranking selected subset of items (FR):} The environment returns a full ranking $\bsigma \in \bSigma_{S}$, drawn from the probability distribution
	$
	\label{eq:prob_rnk1}
	Pr(\bsigma = \sigma|S) = \prod_{i = 1}^{|S|} Pr(X_{\sigma(i)} > X_{\sigma(j)}, ~\forall j \in \{i+1, \ldots |S|\}), \, \forall \bsigma \in \bSigma_S.
	$ 
	In fact, this is equivalent to picking $\bsigma(1)$ according to the winner feedback from $S$, then picking $\bsigma(2)$ from $S \setminus \{\bsigma(1)\}$ following the same feedback model, and so on, until all elements from $S$ are exhausted, or, in other words, successively sampling $|S|$ winners from $S$ according to the \rumk\, model, without replacement.
		
\end{itemize}

\subsection{PAC Performance Objective: Correctness and Sample Complexity} 
\label{sec:obj}

For a \rumk\, instance with $n \geq k$ arms, an arm $i \in [n]$ is said to be $\epsilon$-optimal if $\theta_i > \theta_1 - \epsilon$. A sequential\footnote{We essentially mean a causal algorithm that makes present decisions using only past observed information at each time; the technical details for defining this precisely are omitted.} learning algorithm that depends on feedback from an appropriate subset-wise feedback model is said to be $(\epsilon,\delta)$-{PAC}, for given constants $0 < \epsilon \leq \frac{1}{2}, 0 < \delta \leq 1$, if the following properties hold when it is run on any instance \rumk:
%
(a) it stops and outputs an arm $I \in [n]$ after a finite number of decision rounds (subset plays) with probability $1$, and (b) the probability that its output $I$ is an $\epsilon$-optimal arm in \rumk\, is at least $1-\delta$, i.e, $Pr(\text{$I$ is $\epsilon$-optimal}) \geq 1-\delta$. Furthermore, by {\em sample complexity} of the algorithm, we mean the expected time (number of decision rounds) taken by the algorithm to stop when run on the instance \rumk.



\section{Connecting Subsetwise preferences to Pairwise Scores}
\label{sec:pair_pref}

In this section, we introduce the key concept of Advantage ratio as a means to systematically relate subsetwise preference observations to pairwise scores in general RUMs.
%

Consider any set $S \subseteq [n], \, |S| = k$, and recall that the probability of item $i$ winning in $S$ is $Pr(i|S) := Pr(X_i > X_j, ~\forall j \in [n]\sm \{i\} )$ for all $i \in S,\, S \subseteq [n]$. For any two items $i,j \in [n]$, let us denote $\Delta_{ij} = (\theta_i - \theta_j)$. Let us also denote by $f(\cdot), F(\cdot)$ and $\bar F(\cdot)$ the probability density function\footnote{We assume by default that all noise distributions have a density; the extension to more general noise distributions is left to future work.}, cumulative distribution function and complementary cumulative distribution function of the noise distribution $\cD$, respectively; thus, $F(x) = \int_{-\infty}^{x}f(x)dx$ for any $x \in $ Support$(\cD)$ and $\bar F(x) = \int_{x}^{\infty}f(x)dx = 1-F(x)$ for any $x \in $ Support$(\cD)$.

We now introduce and analyse the \ratio\, (Def. \ref{def:rat}); we will see in Sec. \ref{sec:algo_wi} how this quantity helps us deriving an improved sample complexity guarantee for our $(\epsilon,\delta)$-{PAC} item identification problem. 

\begin{defn}[Advantage ratio and Minimum advantage ratio]
\label{def:rat}
Given any subsetwise preference model defined on $n$ items, we define the advantage ratio of item $i$ over item $j$ within the subset $S \subseteq [n]$, $i,j \in S$ as
\begin{equation*}
\text{\ratio}(i,j,S) = \frac{Pr(i|S)}{Pr(j|S)}. 
\end{equation*}

Moreover, given a playable subset size $k$, we define the minimum advantage ratio, \mrat, of item-$i$ over $j$, as the least advantage ratio of $i$ over $j$ across size-$k$ subsets of $[n]$, i.e., 
\begin{equation}
\label{eq:mrat}
\text{\mrat}(i,j) = \min_{S \subseteq [n],  |S| = k, S \ni i,j}\frac{Pr(i|S)}{Pr(j|S)}.
\end{equation}
\end{defn}

The key intuition here is that when $\text{\mrat}(i,j)$ does not equal $1$, it serves as a distinctive measure for identifying item $i$ and $j$ separately irrespective of the context $S$. We specifically build on this intuition later in Sec. \ref{sec:algo_wi} to propose a new algorithm (Alg. \ref{alg:pp}) which finds the $(\epsilon,\delta)$-PAC best item relying on the unique distinctive properly of the best-item $\theta_1 > \theta_j \forall j \in [n]\sm\{1\}$ (as described in Sec. \ref{sec:prob}).

The following result shows a variational lower bound, in terms of the noise distribution, for the minimum advantage ratio in a \rumk\, model with independent and identically distributed (iid) noise variables, that is often amenable to explicit calculation/bounding. %
%
%


\begin{figure}[t]
	\begin{center}
	\hspace{-15pt}
		\includegraphics[trim={0cm 1pt 0cm 0},clip,scale=0.4,width=0.5\textwidth]{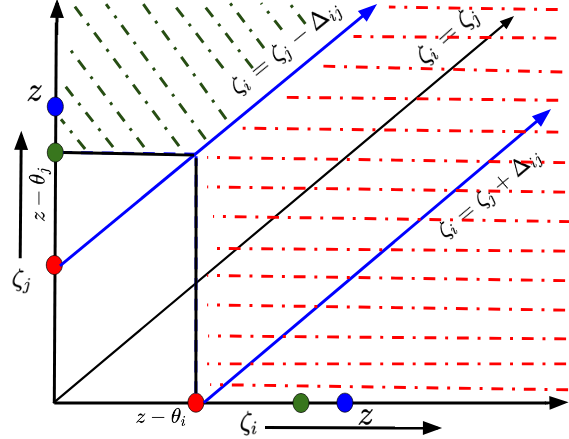}
				\vspace{-5pt}
		\caption{The geometrical interpretation behind \mrat$(i,j)$ for a fixed $z$: The green shaded region is where $X_j > \max(X_i,z)$, the red shaded region is where $X_i > \max(X_j,z)$, and the white rectangle is where $\max(X_i,X_j) < z$. Note how the shape of the green and red region varies as $z$ (blue dot) moves on the real line $\R$ (X-axis).}
		\label{fig:ar_ij}
		\vspace{-5pt}
	\end{center}
\end{figure}
\vspace{0pt}

\begin{restatable}[Variational lower bound for the advantage ratio]{lem}{thmrat}
\label{thm:ratio}
For any \rumk\, based subsetwise preference model and any item pair $(i,j)$,\footnote{We assume $\frac{0}{0}$ to be $\infty$ in the right hand side of Eqn. \ref{eq:mat_z}.} 
\begin{equation}
\label{eq:mat_z}
\text{\mrat}(i,j) \ge \min_{z \in \R}\frac{Pr\big(X_i > \max(X_j,z))}{Pr(X_j > \max(X_i,z)\big)}. 
\end{equation}


Moreover for \rumk\, models one can show that for any triplet $(i,j,S)$, $Pr\big(X_i > \max(X_j,z)) = F(z-\theta_j)\bar F(z-\theta_i) + \int_{z - \theta_j}^{\infty}\bar F(x - \Delta_{ij})f(x)dx$, which further lower bounds \mrat$(i,j)$ by:
\begin{equation*}
{
\label{eq:mat_z2}
\min_{z \in \R}\frac{F(z-\theta_j)\bar F(z-\theta_i) + \int_{z - \theta_j}^{\infty}\bar F(x - \Delta_{ij})f(x)dx}{F(z-\theta_i)\bar F(z-\theta_j) + \int_{z - \theta_i}^{\infty}\bar F(x + \Delta_{ij})f(x)dx}}.
\end{equation*}
\end{restatable}

The proof of the result appears in Appendix \ref{app:ar_z}.
Fig. \ref{fig:ar_ij} shows a geometrical interpretation behind \mrat$(i,j)$, under the joint realization of the pair of values $(\zeta_i,\zeta_j)$.

\begin{rem}
Suppose $\bar S := \arg\min_{|S| = k, i,j \in S}\frac{Pr(i|S)}{Pr(j|S)}$. It is sufficient to consider the domain of $z$ in the right hand side of (\ref{eq:mat_z}) to be just the set $\max_{r \in \bar S \sm \{i,j\}}\theta_r$ + support$(\cD)$, as the proof of Lemma \ref{thm:ratio} brings out. However, for simplicity we use a smaller lower bound in Eqn. \ref{eq:mat_z} and take $z \in R$.
\end{rem}


We next derive the \mrat$(i,j)$\ values certain specific noise distributions:

\begin{restatable}[Analysing \mrat\, for specific noise models]{lem}{lemrat}
\label{lem:ratio}
Given a fixed item pair $(i,j)$ such that $\theta_i > \theta_j$, the following bounds hold under the respective noise models in an iid RUM.

\begin{enumerate}
\item \xn($\lambda$): \mrat$(i,j) \ge e^{\Delta_{ij}} > 1 + \Delta_{ij}$ for Exponential noise with $\lambda=1$. 
\item \textit{Extreme value distribution}$(\mu,\sigma,\chi)$: For \gn$(\mu,\sigma)$ ($\chi = 0$) noise, \mrat$(i,j) = e^{\frac{\Delta_{ij}}{\sigma}} > 1 + \frac{\Delta_{ij}}{\sigma}$. 
\item \un$(a,b)$: \mrat$(i,j) \ge 1 +  \frac{2\Delta_{ij}}{b-a}$ for \un$(a,b)$ noise ($a,b \in \R, b > a,$ and $\Delta_{ij} < \frac{a}{2}$).
\item \textit{Gamma}$(k,\xi)$: \mrat$(i,j) \ge 1 + \Delta_{ij}$ for Gamma$(2,1)$ noise.
\item \textit{Weibull}$(\lambda,k)$: \mrat$(i,j) \ge e^{\lambda\Delta_{ij}} > 1 + \lambda\Delta_{ij}$ for $(k=1)$.
\item \textit{Normal} $\mathcal{N}(0,1)$: For $\Delta_{ij}$ small enough (in a neighborhood of $0$), \mrat$(i,j) \ge 1 + \frac{4}{3}\Delta_{ij}$.
\end{enumerate}
\end{restatable} 

Proof is given in Appendix \ref{app:ar_cases}.


\section{An optimal algorithm for the  winner feedback model}
\label{sec:wi}
In this section, we propose an algorithm (\algpp, Algorithm  \ref{alg:pp}) for the $(\epsilon,\delta)$-{PAC} objective with winner feedback. We then analyse its correctness and sample complexity guarantee (Theorem  \ref{thm:alg_pp}) for any noise distribution $\cD$ (under a mild assumption of its being \mrat\, bounded away from $1$). Following this, we also prove a matching lower bound for the problem which shows that the sample complexity of Algorithm \algpp\, is unimprovable (up to a factor of $\log k$).

\subsection{The \algpp\, algorithm}
\label{sec:algo_wi}
Our algorithm is based on the simple idea of dividing the set of $n$ items into sub-groups of size $k$, querying each subgroup `sufficiently enough', retaining thereafter only the empirically `{strongest item}' of each sub-group, and recursing on the remaining set of items until only one item remains.

More specifically, it starts by partitioning the initial item pool into $G: = \lceil \frac{n}{k} \rceil$ mutually exclusive and exhaustive sets $\cG_1, \cG_2, \cdots \cG_G$ such that $\cup_{j = 1}^{G}\cG_j = S$ and $\cG_{j} \cap \cG_{j'} = \emptyset, ~\forall j,j' \in [G]\, |G_j| = k,\, \forall j \in [G-1]$. 
Each set $\cG_g, \, g \in [G]$ is then queried for $t= O\Big( \frac{k}{\epsilon_\ell^2}\ln \frac{k}{\delta_\ell} \Big)$ rounds, and only the `empirical winner' $c_g$ of each group $g$ is retained in a set $S$, rest are discarded. The algorithm next recurses the same procedure on the remaining set of surviving items, until a single item is left, which then is declared to be the $(\epsilon,\delta)$ {PAC-best} item. Algorithm \ref{alg:pp} presents the pseudocode in more detail.

\noindent {\bf Key idea:} 
The primary novelty here is how the algorithm reasons about the `{strongest item}' in each sub-group $\cG_g$: It maintains the pairwise preferences of every item pair $(i,j)$ in any sub-group $\cG_g$ and simply chooses the item that beats the rest of the items in the sub-group with a positive advantage of greater than $\frac{1}{2}$ (alternatively, the item that wins maximum number of subset-wise plays). Our idea of maintaining pairwise preferences is motivated by a similar algorithm proposed in \cite{SGwin18}; however, their performance guarantee applies to only the very specific class of Plackett-Luce feedback models, whereas the novelty of our current analysis reveals the power of maintaining pairwise-estimates for more general \rumk\, subsetwise model (which includes the Plackett-Luce choice model as a special case). The pseudo code of \algpp\, is given in Alg. \ref{alg:pp}.

\begin{center}
\begin{algorithm}[t]
   \caption{\textbf{\algpp}(\algp)}
   \label{alg:pp}
\begin{algorithmic}[1]
   \STATE {\bfseries Input:} 
   \STATE ~~~ Set of items: $[n]$, Subset size: $n \geq k > 1$
   \STATE ~~~ Error bias: $\epsilon >0$, Confidence parameter: $\delta >0$
   \STATE ~~~ Noise model $(\cD)$ dependent constant $c > 0$
   \STATE {\bfseries Initialize:} 
   \STATE ~~~ $S \leftarrow [n]$, $\epsilon_0 \leftarrow \frac{c\epsilon}{8}$, and $\delta_0 \leftarrow \frac{\delta}{2}$  
   \STATE ~~~ Divide $S$ into $G: = \lceil \frac{n}{k} \rceil$ sets $\cG_1, \cG_2, \ldots, \cG_G$ such that $\cup_{j = 1}^{G}\cG_j = S$ and $\cG_{j} \cap \cG_{j'} = \emptyset, ~\forall j,j' \in [G]$, where $|G_j| = k,\, \forall j \in [G-1]$
   \STATE ~~~ \textbf{If} $|\cG_{G}| < k$, \textbf{then} set $\cR_1 \leftarrow \cG_G$  and $G = G-1$
   \WHILE{$\ell = 1,2, \ldots$}
   \STATE Set $S \leftarrow \emptyset$, $\delta_\ell \leftarrow \frac{\delta_{\ell-1}}{2}, \epsilon_\ell \leftarrow \frac{3}{4}\epsilon_{\ell-1}$
   \FOR {$g = 1,2, \ldots, G$}
   \STATE Play the set $\cG_g$ for $t:= \big\lceil \frac{k}{2\epsilon_\ell^2}\ln \frac{k}{\delta_\ell} \big\rceil$ rounds
   \STATE $w_i \leftarrow$ Number of times $i$ won in $t$ plays of $\cG_g$, $\forall i \in \cG_g$
   \STATE Set $c_g \leftarrow \underset{i \in \cA}{{\arg\max}}~w_i$ and $S \leftarrow S \cup \{c_g\}$
   \ENDFOR
   \STATE $S \leftarrow S \cup \cR_\ell$
      \IF{$(|S| == 1)$}
   \STATE Break (go out of the while loop)
   \ELSIF{$|S|\le k$}
   \STATE $S' \leftarrow $ Randomly sample $k-|S|$ items from $[n] \setminus S$, and $S \leftarrow S \cup S'$, $\epsilon_\ell \leftarrow \frac{c\epsilon}{2}$, $\delta_\ell \leftarrow {\delta}$  
   \ELSE
   \STATE Divide $S$ into $G: = \lceil \frac{|S|}{k} \rceil$ sets $\cG_1, \cG_2, \ldots, \cG_G$, such that $\cup_{j = 1}^{G}\cG_j = S$, and $\cG_{j} \cap \cG_{j'} = \emptyset, ~\forall j,j' \in [G]$, where $|G_j| = k,\, \forall j \in [G-1]$
   \STATE \textbf{If} $|\cG_{G}| < k$, \textbf{then} set $\cR_{\ell+1} \leftarrow \cG_G$  and $G = G-1$
   \ENDIF
   \ENDWHILE
   \STATE {\bfseries Output:} The unique item left in $S$
\end{algorithmic}
\end{algorithm}
\vspace{-2pt}
\end{center}

The following is our chief result; it  proves correctness and a sample complexity bound for Algorithm \ref{alg:pp}.

\begin{restatable}[\algpp: \hspace*{-5pt} Correctness and Sample Complexity]{thm}{ubpp}
\label{thm:alg_pp}
\hspace*{-7pt}
Consider any iid subsetwise preference model \rumk\, based on a noise distribution $\cD$, and suppose that for any item pair $i,j$, we have \mrat$(i,j) \ge 1 + \frac{4c\Delta_{ij}}{1-2c}$ for some $\cD$-dependent constant $c > 0$. Then, Algorithm \ref{alg:pp}, with input constant $c > 0$, is an $(\epsilon,\delta)$-{PAC} algorithm with sample complexity $O(\frac{n}{c^2\epsilon^2} \log \frac{k}{\delta})$.
\end{restatable}

The proof of the result appears in Appendix \ref{app:alg_pm_thm}.

\begin{rem}
The linear dependence on the total number of items, $n$, is, in effect, indicates the price to pay for learning the $n$ unknown model parameters $\btheta = (\theta_1, \ldots, \theta_n)$  which decide the subsetwise winning probabilities of the $n$ items. Remarkably, however, the theorem shows that the PAC sample complexity of the $(\epsilon,\delta)$-best item identification problem, with only winner feedback information from $k$-size subsets, is independent of $k$. One may expect to see improved sample complexity as the number of items being simultaneously tested in each round is large $(k \ge 2)$, but note that on the other side, the sample complexity could also worsen, since it is also harder for a good item to win and show itself in a few draws against a large population of $k-1$ other competitors -- these effects roughly balance each other out, and the final sample complexity only depends on the total number of items $n$ and the accuracy parameters $(\epsilon,\delta)$. 
\end{rem}

Note that Lemma \ref{lem:ratio} gives specific values of the noise-model $\cD$ dependent constant $c > 0$, using which we can derive specific sample complexity bounds for certain noise models: 

\begin{restatable}[Model specific correctness and sample complexity guarantees]{cor}{corpp}
\label{cor:alg_pp} 
For the following representative noise distributions: Exponential$(1)$, Gumbel$(\mu,\sigma)$ Gamma$(2,1)$, Uniform$(a,b)$, Weibull$(\lambda,1)$, Standard normal or Normal$(0,1)$, \algp\, (Alg.\ref{alg:pp}) finds an $(\epsilon,\delta)$-{PAC} item within sample complexity $O\big( \frac{n}{\epsilon^2}\ln \frac{k}{\delta} \big)$.
\end{restatable}

\begin{proof}[Proof sketch]
The proof follows from the general performance guarantee of \algp\, (Thm. \ref{thm:alg_pp}) and Lem. \ref{lem:ratio}. More specifically from Lem. \ref{lem:ratio} it follows that the value of $c$ for these specific distributions are constant, which concludes the claim. For completeness the distribution-specific values of $c$ are given in Appendix \ref{app:alg_pp_cor}.
\end{proof} 

\subsection{Sample Complexity Lower Bound}

In this section we derive a sample complexity lower bound for any $(\epsilon,\delta)$-PAC algorithm for any \rumk\, model with \mrat$(i,j)$ strictly bounded away from $1$ in terms of $\Delta_{ij}$. Our formal claim goes as follows:

\begin{restatable}[Sample Complexity Lower Bound for \rumk\, model]{thm}{lbwin}
\label{thm:lb_pacpl_win}
Given $\epsilon \in (0,\frac{1}{4}]$, $\delta \in (0,1]$, $c > 0$  and an $(\epsilon,\delta)$-PAC algorithm $A$ with winner item feedback, there exists a \rumk\, instance $\nu$ with \mrat$(i,j) \ge 1 + 4c\Delta_{ij}$ for all $i,j \in [n]$, where the expected sample complexity of $A$ on $\nu$ is  at least
$\Omega\big( \frac{n}{c^2\epsilon^2} \ln \frac{1}{2.4\delta}\big)
$.
\end{restatable}

The proof is given in Appendix \ref{app:alg_pm_lm}. It essentially involves a change of measure argument demonstrating a family of Plackett-Luce models (iid Gumbel noise), with the appropriate $c$ value, that cannot easily be teased apart by any learning algorithm.  

Comparing this result with the performance guarantee of our proposed algorithm (Theorem \ref{thm:lb_pacpl_win}) shows that the sample complexity of the algorithm is order-wise optimal (up to a $\log k$ factor). Moreover, this result also shows that the IIA (independence of irrelevant attributes) property of the Plackett-Luce choice model is not essential for exploiting pairwise preferences via rank breaking, as was claimed in \cite{SGwin18}. Indeed, except for the case of \gn\, noise, none of the \rumk\, based models in Corollary \ref{cor:alg_pp} satisfies IIA, but they all respect the  $O\Big( \frac{n}{\epsilon^2}\ln \frac{1}{\delta} \Big)$ $(\epsilon,\delta)$-PAC sample complexity guarantee.

\begin{rem}
For constant $c = O(1)$, the fundamental sample complexity bound of Theorem \ref{thm:lb_pacpl_win} resembles that of PAC best arm identification in the standard multi-armed bandit (MAB) problem \cite{Even+06}. Recall that our problem objective is exactly same as MAB, however our feedback model is very different since in MAB, the learner gets to see the noisy rewards/scores (i.e. the exact values of $X_i$, which can be seen as a noisy feedback of the true reward/score $\theta_i$ of item-$i$), whereas here the learner only sees a $k$-wise relative preference feedback based on the underlying observed values of $X_i$, which is a more indirect way of giving feedback on the item scores, and thus intuitively our problem objective is at least as hard as that of MAB setup. 
\end{rem}

\section{Results for Top-$m$ Ranking (TR) feedback model}
\label{sec:tr}

We now address our $(\epsilon,\delta)$-PAC item identification problem for the case of more general, top-$m$ rank ordered feedback for the \rumk\, model, that generalises both the winner-item (WI) and full ranking (FR) feedback models.
	
\textbf{Top-$m$ ranking of items (TR-$m$):}  In this feedback setting, the environment is assumed to return a ranking of only $m$ items from among $S$, i.e., the environment first draws a full ranking $\bsigma$ over $S$ according to \rumk\, as in {\bf FR} above, and returns the first $m$ rank elements of $\bsigma$, i.e., $(\bsigma(1), \ldots, \bsigma(m))$. It can be seen that for each permutation $\sigma$ on a subset $S_m \subset S$, $|S_m| = m$, we must have $Pr(\bsigma = \sigma|S) = \prod_{i = 1}^{m}Pr(X_{\sigma(i)} > X_{\sigma(j)}, ~\forall j \in \{i+1, \ldots m\}), \, \forall \bsigma \in \bSigma_S^m$, where by $\bSigma_S^m$ we denote the set of all possible $m$-length ranking of items in set $S$, it is easy to note that $|S| = \binom k m m!$. Thus, generating such a $\bsigma$ is also equivalent to successively sampling $m$ winners from $S$ according to the PL model, without replacement. It follows that {\bf TR} reduces to {\bf FR} when $m=k=|S|$ and to {\bf WI} when $m = 1$. Note that the idea for top-$m$ ranking feedback was introduced by \cite{SGrank18} but only for the specific Plackett Luce choice model.

\subsection{Algorithm for top-$m$ ranking feedback}
\label{sec:tr_alg}
In this section, we extend the algorithm proposed earlier  (Alg. \ref{alg:pp}) to handle feedback from the general  top-$m$ ranking feedback model. Based of the performance analysis of our algorithm (Thm. \ref{thm:alg_ppm}), we are able to show that we can achieve an $\frac{1}{m}$-factor improved sample complexity rate with top-$m$ ranking feedback. We finally also give a lower bound analysis under this general feedback model (Thm. \ref{thm:lb_pacpl_rank}) showing the fundamental performance limit of the current problem of interest. Our derived lower bound shows optimality of our proposed algorithm \algpm\, up to logarithmic factors.

\textbf{Main idea:} Same as \algp, the algorithm proposed in this section (Alg. \ref{alg:ppm}) in principle follows the same sequential elimination based strategy to find the near-best item of the \rumk\, model based on pairwise preferences. However, we use the idea of {\em rank breaking}  \citep{AzariRB+14, SGrank18} to extract the pairwise preferences: formally, given any set $S$ of size $k$, if $\bsigma \in \bSigma_{S}^m,\, (S_m \subseteq S,\, |S_m|=m)$ denotes a possible top-$m$ ranking of $S$, then the \rb\, subroutine considers each item in $S$ to be beaten by its preceding items in $\bsigma$ in a pairwise sense. For instance, given a full ranking of a set of $4$ elements $S = \{a,b,c,d\}$, say $b \succ a \succ c \succ d$, Rank-Breaking generates the set of $6$ pairwise comparisons: $\{(b\succ a), (b\succ c), (b\succ d), (a\succ c), (a\succ d), (c\succ d)\}$ etc. 

As a whole, our new algorithm now again divides the set of $n$ items into small groups of size $k$, say $\cG_1, \ldots \cG_G, \, G = \lceil \frac{n}{k} \rceil $, and play each sub-group  some $t = O \Big(\frac{k}{m\epsilon^2} \ln \frac{1}{\delta}\Big )$ many rounds. Inside any fixed subgroup $\cG_g$, after each round of play, it uses \rb\, on the top-$m$ ranking feedback $\bsigma \in \bSigma_{\cG_g}^m$, to extract out $\binom m 2 + (k-m)m$ many pairwise feedback, which is further used to estimate the empirical pairwise preferences $\hp_{ij}$ for each pair of items $i,j \in \cG_g$. Based on these pairwise estimates it then only retains the strongest item of $\cG_g$ and recurse the same procedure on the set of surviving items, until just one item is left in the set. The complete algorithm is given in Alg. \ref{alg:ppm} (Appendix \ref{app:alg_ppm_code}).

Theorem \ref{thm:alg_ppm} analyses the correctness and sample complexity bounds of \algpm. Note that the sample complexity bound of \algpm\, with top-$m$ ranking ({TR}) feedback model is $\frac{1}{m}$-times that of the {WI} model (Thm. \ref{thm:alg_pp}). This is justified since intuitively revealing a ranking on $m$ items in a $k$-set provides about $m$ many WI feedback per round, which essentially leads to the $m$-factor improvement in the sample complexity.  

\begin{restatable}[\algpm (Alg. \ref{alg:ppm}): \hspace*{-5pt} Correctness and Sample Complexity]{thm}{ubppm}
\label{thm:alg_ppm}
\hspace*{-7pt}
Consider any \rumk\, subsetwise preference model based on noise distribution $\cD$ and suppose for any item pair $i,j$,  we have \mrat$(i,j) \ge 1 + \frac{4c\Delta_{ij}}{1-2c}$ for some $\cD$-dependent constant $c > 0$. Then \algpm\, (Alg.\ref{alg:ppm}) with input constant $c > 0$ on top-$m$ ranking feedback model is an $(\epsilon,\delta)$-{PAC} algorithm with sample complexity $O(\frac{n}{mc^2\epsilon^2} \log \frac{k}{\delta})$.
\end{restatable}

Proof is given in Appendix \ref{app:alg_ppm_thm}.

Similar to Cor. \ref{cor:alg_pp}, for the top-$m$ model again, we can derive specific sample complexity bounds for different noise distributions, e.g., \xn, \gn, \nn, \un, \textit{Gamma} etc., in this case as well.

\subsection{Lower Bound: Top-$m$ ranking feedback}
\label{sec:tr_lb}
In this section, we analyze the fundamental limit of sample complexity lower bound for any $(\epsilon,\delta)$-PAC algorithm for \rumk\, model. 

\begin{restatable}[Sample Complexity Lower Bound for \rumk\, model with TR-$m$ feedback]{thm}{lbtr}
\label{thm:lb_pacpl_rank}
Given $\epsilon \in (0,\frac{1}{4}]$ and $\delta \in (0,1]$, and an $(\epsilon,\delta)$-PAC algorithm $A$ with winner item feedback, there exists a \rumk\, instance $\nu$, in which for any pair $i,j \in [n]$ \mrat$(i,j) \ge 1 + 4c\Delta_{ij}$, where the expected sample complexity of $A$ on $\nu$ with top-$m$ ranking feedback has to be at least
$\Omega\bigg( \frac{n}{mc^2\epsilon^2} \ln \frac{1}{2.4\delta}\bigg)
$ for A to be $(\epsilon,\delta)$-PAC.
\end{restatable}

The proof is given in Appendix \ref{app:alg_ppm_lm}.

Similar to the case of winner feedback, comparing Theorem  \ref{thm:alg_ppm} with the above result shows that the sample complexity of \algpm\, is orderwise optimal (up to logarithmic factors), for general case of top-$m$ ranking feedback as well.

\section{Conclusion and Future Directions}
\label{sec:conclusion}

We have identified a new principle to learn with general subset-size preference feedback in general iid RUMs -- rank breaking followed by pairwise comparisons. This has been made possible by extending the concept of pairwise advantage from the popular Plackett-Luce choice model to much more general RUMs, and showing that the IIA property that Plackett-Luce models enjoy is not essential to obtain optimal sample complexity. 

Our results suggest several interesting directions for future investigation, namely the possibility of considering correlated noise models (making the RUM more general), explicitly modeling the dependence of samples on item features or attributes, other performance objectives like regret for online utility optimization, and extension to learning with relative preferences in time-correlated settings like Markov Decision Processes. 


\newpage

\bibliographystyle{plainnat}
\bibliography{bib_pac-rum-battle}

\newpage
\onecolumn
\allowdisplaybreaks

\appendix
{
\section*{\centering \Large{Supplementary for Best-item Learning in Random Utility Models with Subset Choices}}
}
  
\section{Appendix for Section \ref{sec:pair_pref}}  
\label{app:ar}

\subsection{Proof of Lemma \ref{thm:ratio}}  
\label{app:ar_z}

\thmrat*

\begin{proof}
Let us fix any subset $S$ and two consider the items $i,j \in S$ such that $\theta_i > \theta_j$. Recall that we also denote by $\Delta_{ij} = (\theta_i - \theta_j)$.
Let us define a random variable $X_r^S = \max_{r \in S\sm \{i,j\}}X_r$ that denotes the maximum score value taken by the rest of the items in set $S$. Note that the support of $X_r^S$, say denoted by supp$(X_r^{S}) = \max_{r \in S\sm \{i,j\}}\theta_r + $ supp$(\cD)$.

Let us also denote $\bar S := \arg\min_{S \subseteq [n] \mid |S| = k}\frac{Pr(i|S)}{Pr(j|S)}$. We have: 

\begin{align*}
\text{\mrat}(i,j) & = \frac{Pr(i|\bar S)}{Pr(j|\bar S)} = \frac{Pr(\{X_i > X_j\} \cap \{X_i > X_r \, \forall r \in \bar S\sm \{i,j\}\})}{Pr(\{X_j > X_i\} \cap \{X_j > X_r \, \forall r \in \bar S\sm \{i,j\}\})}\\
& = \frac{Pr(\{X_i > X_j\} \cap \{X_i > X_r^{\bar S} \}\})}{Pr(\{X_j > X_i\} \cap \{X_j > X_r^{\bar S}\})}\\
& = \frac{\int_{\text{supp}X_r^{\bar S}} Pr\big(\{X_i > x\} \cap \{X_i > X_j\}\big) f_{X_r^{\bar S}}(x)dx}{\int_{\text{supp}X_r^{\bar S}} Pr\big(\{X_i > x\} \cap \{X_j > X_i\}\big)f_{X_r^{\bar S}}(x)dx}\\
& = \frac{\int_{\text{supp}X_r^{\bar S}} Pr\big(\{X_i > x\} \cap \{X_j > X_i\}\big)\frac{Pr\big(\{X_i > x\} \cap \{X_i > X_j\}\big)}{Pr\big(\{X_i > x\} \cap \{X_j > X_i\}\big)} f_{X_r^{\bar S}}(x)dx}{\int_{\text{supp}X_r^{\bar S}} Pr\big(\{X_i > x\} \cap \{X_j > X_i\}\big)f_{X_r^{\bar S}}(x)dx}\\
& > \min_{z \in \text{supp}(X_r^{\bar S})}\bigg[ \frac{Pr\big(\{X_i > z\} \cap \{X_i > X_j\}\big)}{Pr\big(\{X_i > z\} \cap \{X_j > X_i\}\big)} \bigg] \frac{\int_{\text{supp}X_r^{\bar S}} Pr\big(\{X_i > x\} \cap \{X_j > X_i\}\big) f_{X_r^{\bar S}}(x)dx}{\int_{\text{supp}X_r^{\bar S}} Pr\big(\{X_i > x\} \cap \{X_j > X_i\}\big)f_{X_r^{\bar S}}(x)dx}\\
& = \min_{z \in \text{supp}(X_r^{\bar S})} \frac{Pr\big(\{X_i > \max(X_j,z)\}\big)}{Pr\big(\{X_j > \max(X_i,z)\}\big)} \\
& > \min_{z \in \R} \frac{Pr\big(\{X_i > \max(X_j,z)\}\big)}{Pr\big(\{X_j > \max(X_i,z)\}\big)} 
\end{align*}

Let us now introduce a random variable $Y = \max(X_j,z)$.
Now owing to the \iidn\, assumption of the \rumk\, model, we can further show that:

\begin{align*}
Pr\big(X_i > \max(X_j,z)\big) & = Pr(X_i > Y) = Pr(\{X_i > Y\} \cap \{Y = z\}) + Pr(\{X_i > Y\} \cap \{Y > z\})\\
& = Pr\big(\{X_i > z\} \mid \{Y = z\}\big)Pr(X_j < z) + Pr\big(\{X_i > Y\} \cap \{Y > z\}\big)\\
& = Pr\big(\{\zeta_i + \theta_i > z\} \big)Pr(\zeta_j + \theta_j < z) + Pr\big(\{X_i > X_j\} \cap \{X_j > z\}\big)\\
& = Pr\big(\{\zeta_i > z - \theta_i\} \big)Pr(\zeta_j < z - \theta_j) + Pr\big(\{\zeta_i > \zeta_j - (\theta_i - \theta_j) \} \cap \{\zeta_j > z -\theta_j \}\big)\\
& = F(z-\theta_j)\bar F(z-\theta_i) + \int_{z - \theta_j}^{\infty}\bar F(x - \Delta_{ij})f(x)dx,
\end{align*}
which proves the claim.
\end{proof}

\subsection{Proof of Lemma \ref{lem:ratio}}  
\label{app:ar_cases}

\lemrat*

\begin{proof} We can derive the \mrat$(i,j)$ values for the following distributions by simply applying the lower bound formula stated in Thm. \ref{thm:ratio} \bigg($\min_{z \in \R}\frac{F(z-\theta_j)\bar F(z-\theta_i) + \int_{z - \theta_j}^{\infty}\bar F(x - \Delta_{ij})f(x)dx}{F(z-\theta_i)\bar F(z-\theta_j) + \int_{z - \theta_i}^{\infty}\bar F(x + \Delta_{ij})f(x)dx}
$\bigg) along with their specific density functions as stated below for each specific distributions:

{\bf 1. Exponential noise:}

When the noise distribution $\cD$ is \xn$(1)$, i.e. $\zeta_i, \zeta_j \stackrel{iid}{\sim}$ \xn$(1)$ note that: $f(x) = e^{-x}$, $F(x) = 1-e^{-x}$, and support$(\cD) = [0,\infty)$.

{\bf 2. Gumbel noise:}

When the noise distribution $\cD$ is \gn$(\mu,\sigma)$, i.e. $\zeta_i, \zeta_j \stackrel{iid}{\sim}$ \gn$(\mu,\sigma)$ note that: $f(x) = e^{-\frac{(x-\mu)}{\sigma}}e^{-e^{-\frac{(x-\mu)}{\sigma}}}$, 
$F(x) = e^{-e^{-\frac{(x-\mu)}{\sigma}}}$, and support$(\cD) = (-\infty,\infty)$.

{\bf 3. Uniform noise case:}

When the noise distribution $\cD$ is \un$(a,b)$, i.e. $\zeta_i, \zeta_j \stackrel{iid}{\sim}$ \un$(a,b)$ note that: $f(x) = \frac{1}{b-a}$, 
$F(x) = \frac{x-a}{b-a}$, and support$(\cD) = [a,b]$.

{\bf 4. Gamma noise:}

When the noise distribution $\cD$ is \textit{Gamma}$(k,\xi)$, with $k = 2$ and $\xi = 1$, i.e. $\zeta_i, \zeta_j \stackrel{iid}{\sim}$ \textit{Gamma}$(2,1)$ note that: $f(x) = xe^{-x}$, 
$F(x) = 1 - e^{-x} - xe^{-x}$, and support$(\cD) = [0,\infty)$.

{\bf 5. Weibull noise:}

When the noise distribution $\cD$ is \textit{Weibull}$(\lambda,k)$, with $k = 1$, i.e. $\zeta_i, \zeta_j \stackrel{iid}{\sim}$ \textit{Weibull}$(\lambda,1)$ note that: $f(x) = \frac{1}{\lambda}e^{-\frac{x}{\lambda}}$, 
$F(x) = 1-e^{-\frac{x}{\lambda}}$, and support$(\cD) = [0,\infty)$.

{\bf 6. Argument for the Gaussian noise case.}
Note that Gaussian distributions do not have closed form CDFs and are difficult to compute in general, so we propose a different line of analysis specifically for the Gaussian noise case:
Take the noise distribution to be standard normal, i.e., $\zeta_i, \zeta_j \stackrel{iid}{\sim} \mathcal{N}(0,1)$, with density $f(x) = \frac{1}{\sqrt{2\pi}} e^{-x^2/2}$. When $X_i = \theta_i + \zeta_i$ and $X_j = \theta_j + \zeta_j$ with $\Delta_{ij} = \theta_i - \theta_j > 0
$, we find a lower bound on 
\[ \inf_{z \in \R}\frac{Pr\big(X_i > \max(X_j,z))}{Pr(X_j > \max(X_i,z)\big)}. \]

First, note that by translation, we can take $\theta_j = 0$ and $\theta_i = \Delta$ without loss of generality. Doing so allows us to write
\[ Pr\big(X_i > \max(X_j,z)) = F(z)(1-F(z-\Delta)) + \int_z^\infty (1-F(y-\Delta))f(y) dy \equiv g(\Delta, z), \]
and likewise (taking $X_i = 0, X_j = z-\Delta$),
\[ Pr\big(X_j > \max(X_i,z)) = F(z-\Delta)(1-F(z)) + \int_{z-\Delta}^\infty (1-F(y+\Delta))f(y) dy \equiv g(-\Delta, z-\Delta). \]
With this notation, we wish to minimize the ratio $\frac{g(\Delta,z)}{g(-\Delta, z-\Delta)}$ over $z \in \R$. 

Notice that $g(0, z) = 1/2$ for any $z$, and $\frac{\partial g(\Delta,z)}{\partial \Delta} = F(z)f(z-\Delta) + \int_z^\infty f(y-\Delta)f(y)dy$. 
Hence, up to first order, for $\Delta$ small enough, we have\footnote{The argument can be made rigorous using the Taylor expansion up to 2nd order.}
\begin{align*}
	&\frac{g(\Delta,z)}{g(-\Delta, z-\Delta)} \approx \frac{g(0, z) + \Delta \frac{\partial g(\Delta,z)}{\partial \Delta} \vert_{\Delta = 0} }{g(0, z-\Delta) - \Delta \frac{\partial g(\tilde{\Delta},z-\Delta)}{\partial \tilde{\Delta}} \vert_{\tilde{\Delta} = 0} } \\
	&= \frac{\frac{1}{2} + \Delta F(z) f(z) + \Delta \int_z^\infty f(y)^2 dy}{\frac{1}{2} - \Delta F(z-\Delta) f(z-\Delta) - \Delta \int_{z-\Delta}^\infty f(y)^2 dy} \\
	&\equiv \frac{h_1(z)}{h_2(z)}, \quad  \text{say.}
\end{align*}
Differentiating the ratio and equating it to $0$ to find its minimum, we obtain the condition
\begin{align*}
	h_1'(z_*)h_2(z_*) &= h_1(z_*)h_2'(z_*) \\
	\Leftrightarrow F(z_*)f'(z_*)h_2(z_*) &= -F(z_*-\Delta)f'(z_*-\Delta)h_1(z_*). 
\end{align*}
Assuming $z_* \approx 0 \Rightarrow h_1(z_*) \approx h_2(z_*)$ gives the solution $z_* \approx \frac{\Delta}{2}$, for which
\begin{align*}
	\frac{h_1(z_*)}{h_2(z_*)} &= \frac{\frac{1}{2} + \Delta F(\frac{\Delta}{2}) f(\frac{\Delta}{2}) + \Delta \int_{\frac{\Delta}{2}}^\infty f(y)^2 dy}{\frac{1}{2} - \Delta F(-\frac{\Delta}{2}) f(-\frac{\Delta}{2}) - \Delta \int_{-\frac{\Delta}{2}}^\infty f(y)^2 dy} \\
	&\approx \frac{\frac{1}{2} +  \frac{\Delta}{2} \frac{1}{\sqrt{2\pi}} +  \frac{\Delta}{2}}{\frac{1}{2} -  \frac{\Delta}{2} \frac{1}{\sqrt{2\pi}} -  \frac{\Delta}{2}} \\
	&\geq \left( 1 + \Delta\left( 1 + \frac{1}{\sqrt{2\pi}} \right) \right)^2 \geq 1 + \frac{4\Delta}{3},
\end{align*} 
for $\Delta$ small enough. 

\end{proof}

\section{Appendix for Section \ref{sec:algo_wi}}  

\subsection{Proof of Theorem \ref{thm:alg_pp}}  
\label{app:alg_pm_thm}

\ubpp*

\begin{proof}
We start by analyzing the required sample complexity of \algpp. 
Note that at any iteration $\ell$, any set $\cG_g$ is played for exactly $t= \frac{k}{2\epsilon_\ell^2}\ln \frac{k}{\delta_\ell}$ many number of rounds. Also, since the algorithm discards  exactly $k-1$ items from each set $\cG_g$, the maximum number of iterations possible is $\lceil  \ln_k n \rceil$. Now at any iteration $\ell$, since $G = \Big\lfloor \frac{|S_\ell|}{k} \Big\rfloor < \frac{|S_\ell|}{k}$, the total sample complexity the for iteration is at most $\frac{|S_\ell|}{k}t \le \frac{n}{2k^{\ell-1}\epsilon_\ell^2}\ln \frac{k}{\delta_\ell}$, as $|S_\ell| \le \frac{n}{k^\ell}$ for all $\ell \in [\lfloor  \ln_k n \rfloor]$. 
Also note that for all but last iteration $\ell \in [\lfloor  \ln_k n \rfloor]$, we have $\epsilon_\ell = \frac{c\epsilon}{8}\bigg( \frac{3}{4} \bigg)^{\ell-1}$, and $\delta_\ell = \frac{\delta}{2^{\ell+1}}$.
Moreover, for the last iteration $\ell = \lceil  \ln_k n \rceil$, the sample complexity is clearly $t= \frac{2k}{c^2\epsilon^2}\ln \frac{2k}{\delta}$, as in this case $\epsilon_\ell = \frac{c\epsilon}{2}$, and $\delta_\ell = \frac{\delta}{2}$, and $|S| = k$.
Thus, the total sample complexity of Algorithm \ref{alg:pp} is given by 

\begin{align*}
\sum_{\ell = 1}^{\lceil  \ln_k n \rceil}\frac{|S_\ell|}{2\epsilon_\ell^2}\ln \frac{k}{\delta_\ell}  
& \le \sum_{\ell = 1}^{\infty}\frac{n}{2k^\ell\bigg(\frac{c\epsilon}{8}\big(\frac{3}{4}\big)^{\ell-1}\bigg)^2}k\ln \frac{k 2^{\ell+1}}{\delta} + \frac{2k}{c^2\epsilon^2}\ln \frac{2k}{\delta}\\ 
& \le \frac{64n}{2c^2\epsilon^2}\sum_{\ell = 1}^{\infty}\frac{16^{\ell-1}}{(9k)^{\ell-1}}\Big( \ln \frac{k}{\delta} + {(\ell+1)} \Big) + \frac{2k}{c^2\epsilon^2}\ln \frac{2k}{\delta} \\
& \le \frac{32n}{c^2\epsilon^2}\ln \frac{k}{\delta}\sum_{\ell = 1}^{\infty}\frac{4^{\ell-1}}{(9k)^{\ell-1}}\Big( {3\ell} \Big) + \frac{2k}{c^2\epsilon^2}\ln \frac{2k}{\delta} = O\bigg(\frac{n}{c^2\epsilon^2}\ln \frac{k}{\delta}\bigg) ~[\text{for any } k > 1],
\end{align*}

and this proves the sample complexity bound of Theorem \ref{thm:alg_pp}. We next prove the $(\epsilon,\delta)$-{PAC} property of \algpp. 

Consider any fixed subgroup $\cG$ of size $k$, such that two items $a,b \in \cG$. Now suppose we denote by $Pr(\{ab\}|\cG) = Pr(a|\cG) + Pr(b|\cG)$ the probability that either $a$ or $b$ wins in the subset $\cG$. Then the probability that $a$ wins in $\cG$ given either $a$ or $b$ won in $\cG$ is given by $p_{ab|\cG} := \frac{Pr(a|\cG)}{Pr(\{ab\}|\cG)} = \frac{Pr(a|\cG)}{Pr(a|\cG) + Pr(b|\cG)}$ --- this quantity in a way models the pairwise preference of $a$ over $b$ in the set $\cG$. Note that as long as $\theta_a > \theta_b$, $p_{ab|\cG} > \frac{1}{2}$, for any $\cG$ (since $Pr(a|\cG) > Pr(b|\cG)$). We in fact now introduce the notation $p_{ab}: = \min_{\cG\subseteq [n] | |\cG|=k}p_{ab|\cG}$.

\begin{lem}
\label{lem:pp_1i}
For any item pair $i,j \in [n]$ and any set $S \subseteq [n]$, if their advantage ratio  $\frac{Pr(i|S)}{Pr(j|S)} \ge 1 + \alpha$, for some $\alpha > 0$, then pairwise preference of item $i$ over $j$ in set $S$  $p_{ij|S} > \frac{1}{2} + \frac{\alpha}{4}$.
\end{lem}

\begin{proof}
Note that 

\begin{align*}
& \frac{Pr(i|S)}{Pr(j|S)} \ge 1 + \alpha 
\implies \frac{Pr(i|S) - Pr(j|S)}{Pr(j|S)} \ge \alpha\\
\implies & p_{ij|S} - 0.5 = \frac{Pr(i|S) - Pr(j|S)}{2(Pr(i|S) + Pr(j|S))} \ge \frac{\alpha Pr(j|S))}{2(Pr(j|S) + Pr(j|S))} = \frac{\alpha}{4},
\end{align*}
which concludes the proof.
\end{proof}

\begin{cor}
\label{cor:pp_1i}
For any item pair $i,j \in [n]$, if
If \mrat$(i,j) \ge 1 + \alpha $ for some $\alpha > 0$, then $p_{ij} > \frac{1}{2} + \frac{\alpha}{4}$.
\end{cor}

\begin{proof}
The proof directly follows from Lem . \ref{lem:pp_1i} by using subset $S = \min_{S \subseteq [n] \mid |S| = k}$\mrat$(i,j)$.
\end{proof}


Let us denote the set of surviving items $S$ at the beginning of phase $\ell$ as $S_\ell$.
We now claim the following crucial lemma which shows at any phase $\ell$, the best (the one with highest $\theta$ parameter) item retained in $S_{\ell + 1}$ can not be too bad in comparison to the best item of $S_{\ell}$. The formal claim goes as follows:

\begin{restatable}[]{lem}{lemdiv}
\label{lem:div_bb} 
At any iteration $\ell$, for any $\cG_g$, if $i_g := \underset{i \in \cG_g}{{\arg\max}}~\theta_i$, then with probability at least $(1-\delta_\ell)$, $\theta_{c_g} > \theta_{i_g} - \frac{\epsilon_\ell}{c}$.
\end{restatable} 

\begin{proof}
Let us define $\hp_{ij} = \frac{w_i}{w_i + w_j}, \, \forall i,j \in \cG_g, i \neq j$. Then clearly $\hp_{c_gi_g} \ge \frac{1}{2}$, as 
$c_g$ is the empirical winner in $t$ rounds, i.e. $c_g \leftarrow \underset{i \in \cG_g}{{\arg\max}}~w_i$. Moreover $c_g$ being the empirical winner of $\cG_g$ we also have $w_{c_g} \ge \frac{t}{k}$, and thus $w_{c_g} + w_{r_g} \ge \frac{t}{k}$ as well. Let $n_{ij}:= w_i + w_j$ denotes the number of pairwise comparisons of item $i$ and $j$ in $t$ rounds, $i,j \in \cG_g$. Clearly $0 \le n_{ij} \le t$. Then let us analyze the probability of a `bad event' where $c_g$ is indeed such that $\theta_{c_g} < \theta_{i_g} - \frac{\epsilon_\ell}{c}$. 

This implies that the advantage ratio of $i_g$ and $c_g$ in $\cG$ is $\frac{Pr(i_g|\cG)}{Pr(c_g|\cG)} \ge 1 + 4\epsilon_\ell$.

But now by Lem. \ref{lem:pp_1i} this further implies $p_{i_g c_g |\cG}  \ge \frac{1}{2} + \epsilon_\ell$. But since $c_g$ beats $i_g$ empirically in the subgroup $\cG$, this implies $\hat p_{c_g i_g} > \frac{1}{2} $. The following argument shows that this is even unlikely to happen, more formally with probability $(1- \delta_\ell/k)$:

\begin{align*}
& Pr\Big( \big\{ \hp_{c_gi_g} \ge \frac{1}{2} \big\} \Big)\\
& = Pr\Big( \big\{ \hp_{c_gi_g} \ge \frac{1}{2} \big\} \hspace*{-2pt} \cap \hspace*{-2pt} \big\{ n_{c_gi_g} \ge \frac{t}{k} \big\} \Big) \hspace*{-3pt}+ \hspace*{-3pt}{Pr\Big(\big\{ n_{c_gi_g} < \frac{t}{k} \big\}\Big)}Pr\Big( \big\{ \hp_{c_gi_g} \ge \frac{1}{2} \big\} \Big | \big\{ n_{c_gi_g} \hspace*{-2pt}<\hspace*{-2pt} \frac{t}{k} \big\}\Big)\\
& =  Pr\Big( \big\{ \hp_{c_gi_g} - \epsilon_\ell \ge \frac{1}{2} - \epsilon_\ell \big\} \cap \big\{ n_{c_gi_g} \ge \frac{t}{k} \big\} \Big)\\
& \le Pr\Big( \big\{ \hp_{c_gi_g} - p_{c_g i_g | \cG} \ge {\epsilon_\ell} \big\} \cap \big\{ n_{c_gi_g} \ge \frac{t}{k} \big\} \Big)\\
& \le \exp\Big( -2\dfrac{t}{k}\big({\epsilon_\ell}\big)^2 \Big) = \frac{\delta_\ell}{k}.
\end{align*}

where the first inequality holds as $p_{c_gi_g | \cG} < \frac{1}{2} - \epsilon_\ell$, and the second inequality follows from Hoeffdings lemma.
Now taking the union bound over all $\epsilon_\ell$-suboptimal elements $i'$ of $\cG_g$ (i.e. $\theta_{i'} < \theta_{i_g} - \epsilon_\ell$), we get: 
\[Pr\Big(\big\{\exists i' \in \cG_g \mid p_{i'i_g} \hspace*{-2pt} < \hspace*{-2pt} \frac{1}{2} - \epsilon_\ell, \text{and } c_g = i' \big\}\Big) \hspace*{-1pt}\le\hspace*{-1pt} \frac{\delta_\ell}{k}\Big \vert\big\{\exists i' \in \cG_g \mid p_{i'i_g} \hspace*{-2pt} < \hspace*{-2pt} \frac{1}{2} - \epsilon_\ell, \text{and } c_g = i' \big\}\Big\vert \le \delta_\ell, 
\]
as $|\cG_g| = k$, and the claim follows henceforth.
\end{proof}

Let us denote the single element remaining in $S$ at termination by $r \in [n]$.
Also note that for the last iteration $\ell = \lceil \ln_k n \rceil$, since $\epsilon_\ell = \frac{\epsilon}{2}$, and $\delta_\ell = \frac{\delta}{2}$, applying Lemma \ref{lem:div_bb} on $S$, we get that $Pr\Big( \theta_{r} < \theta_{i_g} - \frac{\epsilon}{2} \Big) \leq \frac{\delta}{2}$.  

Without loss of generality we assume the best item of the \rumk\, model is $\theta_1$, i.e. $\theta_1 > \theta_i \, \forall i \in [n] \sm \{1\}$.
Now for any iteration $\ell$, let us define $g_\ell \in [G]$ to be the index of the set that contains \textit{best item} of the entire set $S_\ell$, i.e. $\arg\max_{i \in S_\ell}\theta_i \in \cG_{g_\ell}$. Then applying Lemma \ref{lem:div_bb}, with probability at least $(1-\delta_\ell)$,\, $\theta_{c_{g_\ell}} > \theta_{i_{g_\ell}} - \epsilon_\ell/c$. Note that initially, at phase $\ell = 1$, $i_{g_\ell} = 1$. Then, for each iteration $\ell$, applying Lemma \ref{lem:div_bb} recursively to $\cG_{g_\ell}$, we finally get $\theta_{r} > \theta_1 -\Big( \frac{\epsilon}{8} + \frac{\epsilon}{8}\Big(\frac{3}{4}\Big) + \cdots + \frac{\epsilon}{8}\big(\frac{3}{4}\big)^{\lfloor \ln_k n \rfloor} \Big) - \frac{\epsilon}{2} \ge \theta_1 -\frac{\epsilon}{8}\Big( \sum_{i = 0}^{\infty}\big( \frac{3}{4} \big)^{i} \Big) - \frac{\epsilon}{2} \ge \theta_1 - \epsilon$.
Thus assuming the algorithm does not fail in any of the iteration $\ell$, we finally have that $p_{r_*1} > \frac{1}{2} - \epsilon$---this shows that the final item output by \algp\, is $\epsilon$ optimal.

Finally since at any phase $\ell$, the algorithm fails with probability at most $\delta_\ell$, the total failure probability of the algorithm is at most $\Big( \frac{\delta}{4} + \frac{\delta}{8} + \cdots + \frac{\delta}{2^{\lceil \ln_k n  \rceil}} \Big) + \frac{\delta}{2} \le \delta$.
This concludes the correctness of the algorithm showing that it indeed satisfies the $(\epsilon,\delta)$-PAC objective.
\end{proof}

\subsection{Proof of Corollary \ref{cor:alg_pp}}
\label{app:alg_pp_cor}

\begin{proof}
The proof essentially follows from the general performance guarantee of \algp\, (Thm. \ref{thm:alg_pp}) and Lem. \ref{lem:ratio}. More specifically from Lem. \ref{lem:ratio} it follows that the value of $c$ for these specific distributions are constant, which concludes the claim. For completeness the distribution-specific values of $c$ are given below:
\begin{enumerate}
\item $c = 0.25$ for Exponential noise with $\lambda=1$
\item $c = \frac{0.25}{\sigma}$ for \gn$(\mu,\sigma)$
\item $c = \frac{0.5}{(b-a)}$ for \un$(a,b)$
\item $c = \frac{1}{4}$ for \textit{Gamma}$(2,1)$
\item $c = \frac{\lambda}{4}$ for \textit{Weibull}$(\lambda,1)$
\item $c = \frac{1}{3}$ \textit{Normal} $\mathcal{N}(0,1)$, etc.
\end{enumerate}
\end{proof} 

\subsection{Proof of Theorem \ref{thm:lb_pacpl_win}}
\label{app:alg_pm_lm}

Before proving the lower bound result we state a key lemma from \cite{Kaufmann+16_OnComplexity} which is a general result for proving information theoretic lower bound for bandit problems:

Consider a multi-armed bandit (MAB) problem with $n$ arms or actions $\cA = [n]$. At round $t$, let $A_t$ and $Z_t$ denote the arm played and the observation (reward) received, respectively. Let $\cF_t = \sigma(A_1,Z_1,\ldots,A_t,Z_t)$ be the sigma algebra generated by the trajectory of a sequential bandit algorithm up to round $t$.
\begin{restatable}[Lemma $1$, \cite{Kaufmann+16_OnComplexity}]{lem}{gar16}
\label{lem:gar16}
Let $\nu$ and $\nu'$ be two bandit models (assignments of reward distributions to arms), such that $\nu_i ~(\text{resp.} \,\nu'_i)$ is the reward distribution of any arm $i \in \cA$ under bandit model $\nu ~(\text{resp.} \,\nu')$, and such that for all such arms $i$, $\nu_i$ and $\nu'_i$ are mutually absolutely continuous. Then for any almost-surely finite stopping time $\tau$ with respect to $(\cF_t)_t$,
\vspace*{-5pt}
\begin{align*}
\sum_{i = 1}^{n}\E_{\nu}[N_i(\tau)]KL(\nu_i,\nu_i') \ge \sup_{\cE \in \cF_\tau} kl(Pr_{\nu}(\cE),Pr_{\nu'}(\cE)),
\end{align*}
where $kl(x, y) := x \log(\frac{x}{y}) + (1-x) \log(\frac{1-x}{1-y})$ is the binary relative entropy, $N_i(\tau)$ denotes the number of times arm $i$ is played in $\tau$ rounds, and $Pr_{\nu}(\cE)$ and $Pr_{\nu'}(\cE)$ denote the probability of any event $\cE \in \cF_{\tau}$ under bandit models $\nu$ and $\nu'$, respectively.
\end{restatable}

We now proceed to proof our lower bound result of Thm. \ref{thm:lb_pacpl_win}.

\lbwin*

\begin{proof}
In order to apply the change of measure based lemma Lem. \ref{lem:gar16}, we constructed the following specific instances of the \rumk\, model for our purpose and assume $\cD$ to be the \gn$(0,1)$ noise:

\begin{align*}
\text{True Instance} ~(\bnu^1): \theta_j^1 = 1-\epsilon, \forall j \in [n]\setminus \{1\}, \text{ and } \theta_1^1 = 1,
\end{align*}

Note the only $\epsilon$-optimal arm in the true instance is arm $1$. Now for every suboptimal item $a \in [n]\setminus \{1\}$, consider the modified instances $\bnu^a$ such that:
\begin{align*}
\text{Instance--a} ~(\bnu^a): \theta^a_j = 1-2\epsilon, \forall j \in [n]\setminus \{a,1\}, \, \theta_1^a = 1-\epsilon, \text{ and } \theta_a^a = 1.
\end{align*}

For any problem instance $\bnu^a, \, a \in [n]\setminus\{1\}$, the probability distribution associated with arm $S \in \cA$ is given by
\[
\nu^a_S \sim Categorical(p_1, p_2, \ldots, p_k), \text{ where } p_i = Pr(i|S), ~~\forall i \in [k], \, \forall S \in \cA,
\]
where $Pr(i|S)$ is as defined in Section \ref{sec:feed_mod}. 
Note that the only $\epsilon$-optimal arm for \textbf{Instance-a} is arm $a$. Now applying Lemma \ref{lem:gar16}, for any event $\cE \in \cF_\tau$ we get,

\begin{align}
\label{eq:FI_a}
\sum_{\{S \in \cA : a \in S\}}\E_{\bnu^1}[N_S(\tau_A)]KL(\bnu^1_S, \bnu^a_S) \ge {kl(Pr_{\nu}(\cE),Pr_{\nu'}(\cE))}.
\end{align}

The above result holds from the straightforward observation that for any arm $S \in \cA$ with $a \notin S$, $\bnu^1_S$ is same as $\bnu^a_S$, hence $KL(\bnu^1_S, \bnu^a_S)=0$, $\forall S \in \cA, \,a \notin S$. 
For notational convenience, we will henceforth denote $S^a = \{S \in \cA : a \in S\}$. 

Now let us analyse the right hand side of \eqref{eq:FI_a}, for any set $S \in S^a$. 

\textbf{Case-1:} First let us consider $S \in S^a$ such that $1 \notin S$. Note that in this case:
\begin{align*}
\nu^1_S(i) = \frac{1}{k}, \text{ for all } i \in S
\end{align*}

On the other hand, for problem \textbf{Instance-a}, we have that: 

\begin{align*}
\nu^a_S(i) = 
\begin{cases} 
\frac{e^1}{(k-1)e^{1-2\epsilon} + e^1} \text{ when } S(i) = a,\\
\frac{e^{1-2\epsilon}}{(k-1)e^{1-2\epsilon} + e^1}, \text{ otherwise}
\end{cases}
\end{align*}

Now using the following upper bound on $KL(\p_1,\p_2) \le \sum_{x \in \X}\frac{p_1^2(x)}{p_2(x)} -1$, $\p_1$ and $\p_2$ be two probability mass functions on the discrete random variable $\X$ \citep{klub16} we get:

\begin{align*}
KL(\bnu^1_S, \bnu^a_S) & \le (k-1)\frac{(k-1)e^{1-2\epsilon} + e^1}{k^2(e^{1-2\epsilon})} + \frac{(k-1)e^{1-2\epsilon} + e^1}{k^1e^1}-1\\
& = \frac{(k-1)}{k^2}\bigg( e^\epsilon - e^{-\epsilon}\bigg)^2 = \frac{(k-1)}{k^2}e^{-2\epsilon}(e^\epsilon-1)^2 \le \frac{\epsilon^2}{k} \text{ for any } \epsilon \in \bigg[0,\frac{1}{2}\bigg]
\end{align*}

\textbf{Case-2:} Now let us consider the remaining set in $S^a$ such that $S \owns 1,a$. Similar to the earlier case in this case we get that:

\begin{align*}
\nu^a_S(i) = 
\begin{cases} 
\frac{e^1}{(k-1)e^{1-\epsilon} + e^1} \text{ when } S(i) = 1,\\
\frac{e^{1-\epsilon}}{(k-1)e^{1-\epsilon} + e^1}, \text{ otherwise}
\end{cases}
\end{align*}

On the other hand, for problem \textbf{Instance-a}, we have that: 

\begin{align*}
\nu^a_S(i) = 
\begin{cases} 
\frac{e^{1-\epsilon}}{(k-2)e^{1-2\epsilon} + e^{1-\epsilon} + e^1} \text{ when } S(i) = 1,\\
\frac{e^{1}}{(k-2)e^{1-2\epsilon} + e^{1-\epsilon} + e^1} \text{ when } S(i) = a,\\
\frac{e^{1-2\epsilon}}{(k-2)e^{1-2\epsilon} + e^{1-\epsilon} + e^1}, \text{ otherwise}
\end{cases}
\end{align*}

Now using the previously mentioned upper bound on the KL divergence, followed by some elementary calculations one can show that for any $\big[0,\frac{1}{4} \big]$:

\begin{align*}
KL(\bnu^1_S, \bnu^a_S) & \le \frac{8\epsilon^2}{k}
\end{align*}

Thus combining the above two cases we can conclude that for any $S \in S^a$, $KL(\bnu^1_S, \bnu^a_S) \le \frac{8\epsilon^2}{k}$, and as argued above for any $S \notin S^a$, $KL(\bnu^1_S, \bnu^a_S) = 0$.

Note that the only $\epsilon$-optimal arm for any \textbf{Instance-a} is arm $a$, for all $a \in [n]$.
Now, consider $\cE_0 \in \cF_\tau$ be an event such that the algorithm $A$ returns the element $i = 1$, and let us analyse the left hand side of \eqref{eq:FI_a} for $\cE = \cE_0$. Clearly, $A$ being an $(\epsilon,\delta)$-PAC algorithm, we have $Pr_{\bnu^1}(\cE_0) > 1-\delta$, and $Pr_{\bnu^a}(\cE_0) < \delta$, for any suboptimal arm $a \in [n]\setminus\{1\}$. Then we have 

\begin{align}
\label{eq:win_lb2}
kl(Pr_{\bnu^1}(\cE_0),Pr_{\bnu^a}(\cE_0)) \ge kl(1-\delta,\delta) \ge \ln \frac{1}{2.4\delta}
\end{align}

where the last inequality follows from \cite{Kaufmann+16_OnComplexity} (Eqn. $3$).

Now applying \eqref{eq:FI_a} for each modified bandit \textbf{Instance-$\bnu^a$}, and summing over all suboptimal items $a \in [n]\setminus \{1\}$ we get,

\begin{align}
\label{eq:win_lb2.5}
\sum_{a = 2}^{n}\sum_{\{S \in \cA \mid a \in S\}}\E_{\bnu^1}[N_S(\tau_A)]KL(\bnu^1_S,\bnu^a_S) \ge (n-1)\ln \frac{1}{2.4\delta}.
\end{align}

Using the upper bounds on $KL(\bnu^1_S,\bnu^a_S)$ as shown above, the right hand side of \eqref{eq:win_lb2.5} can be further upper bounded as:

\begin{align}
\label{eq:win_lb3}
\nonumber \sum_{a = 2}^{n}&\sum_{\{S \in \cA \mid a \in S\}} \E_{\bnu^1}[N_S(\tau_A)]KL(\bnu^1_S,\bnu^a_S) \le \sum_{S \in \cA}\E_{\bnu^1}[N_S(\tau_A)]\sum_{\{a \in S \mid a \neq 1\}}\frac{8\epsilon^2}{k}\\
& = \sum_{S \in \cA}\E_{\bnu^1}[N_S(\tau_A)]{k - \big(\1(1 \in S)\big)}\frac{8\epsilon^2}{k} \le \sum_{S \in \cA}\E_{\bnu^1}[N_S(\tau_A)]{8\epsilon^2}.
\end{align}

Finally noting that $\tau_A = \sum_{S \in \cA}[N_S(\tau_A)]$, combining \eqref{eq:win_lb2.5} and \eqref{eq:win_lb3}, we get 

\begin{align}
\label{eq:win_lb_fin}
(8\epsilon^2)\E_{\bnu^1}[\tau_A] =  \sum_{S \in \cA}\E_{\bnu^1}[N_S(\tau_A)](8\epsilon^2) \ge (n-1)\ln \frac{1}{2.4\delta}.
\end{align}
Now note that as derived in Lem. \ref{lem:ratio}, for \gn$(0,1)$ noise, we have shown that for any pair $i,j \in [n]$, \mrat$(i,j) = e^{\Delta_{ij}} > 1 + \Delta_{ij} = 1 + 4\frac{1}{4}\Delta_{ij} \implies$ the value of the noise dependent constant $c$ can be taken to be $c = \frac{1}{4}$. Thus rewriting Eqn. \ref{eq:win_lb_fin} we get $\E_{\bnu^1}[\tau_A]  \ge \frac{(n-1)}{8\epsilon^2}\ln \frac{1}{2.4\delta} = \frac{(n-1)}{128c^2\epsilon^2}\ln \frac{1}{2.4\delta}$.
The above construction shows the existence of a problem instance of \rumk\, model where any $(\epsilon,\delta)$-PAC algorithm requires at least $\Omega(\frac{n}{c^2\epsilon^2}\ln \frac{1}{2.4\delta})$ samples to ensure correctness of its performance, concluding our proof.
\end{proof}  

\begin{rem}
It is worth noting that our lower bound analysis is essentially in spirit the same as the one proposed by \cite{SGwin18} for the Plackett luce model. However note that, their PAC objective is quite different than the one considered in our case--precisely their model is positive scale invariant, unlike ours which is shift invariant w.r.t the model parameters $\btheta$. Moreover our setting aims to find a $\epsilon$-best item in additive sense (i.e. to find an item $i$ whose score difference w.r.t to the best item $1$ is at most $\epsilon>0$ or $\theta_1-\theta_i < \epsilon$), as opposed to the $(\epsilon,\delta)$-PAC objective considered in \cite{SGwin18} which seeks to find a multiplicative-$\epsilon$-best item (i.e. to find an item $i$ which matches the score of the best item up to $\epsilon$-factor or $\theta_i > \epsilon\theta_1$). Therefore the problem instance construction for proving a suitable lower bound these two setups are very different where lies the novelty of out current lower bound analysis.
\end{rem}

\section{Appendix for Section \ref{sec:tr}}

\subsection{Pseudo code of \algpp\, for top-$m$ ranking feedback (\algpm)} 
\label{app:alg_ppm_code}

The description is given in Algorithm \ref{alg:ppm}.

\begin{center}
\begin{algorithm}[th]
   \caption{\textbf{\algpp} ({TR}-$m$ feedback)}
   \label{alg:ppm}
\begin{algorithmic}[1]
   \STATE {\bfseries Input:} 
   \STATE ~~~ Set of items: $[n]$, and subset size: $k > 2$ ($n \ge k \ge m$)
   \STATE ~~~ Error bias: $\epsilon >0$, and confidence parameter: $\delta >0$
   \STATE ~~~ Noise model $(\cD)$ dependent constant $c > 0$
   \STATE {\bfseries Initialize:} 
   \STATE ~~~ $S \leftarrow [n]$, $\epsilon_0 \leftarrow \frac{c\epsilon}{8}$, and $\delta_0 \leftarrow \frac{\delta}{2}$  
   \STATE ~~~ Divide $S$ into $G: = \lceil \frac{n}{k} \rceil$ sets $\cG_1, \cG_2, \cdots \cG_G$ such that $\cup_{j = 1}^{G}\cG_j = S$ and $\cG_{j} \cap \cG_{j'} = \emptyset, ~\forall j,j' \in [G], \, |G_j| = k,\, \forall j \in [G-1]$.
    \textbf{If} $|\cG_{G}| < k$, \textbf{then} set $\cR_1 \leftarrow \cG_G$  and $G = G-1$.
   \WHILE{$\ell = 1,2, \ldots$}
   \STATE Set $S \leftarrow \emptyset$, $\delta_\ell \leftarrow \frac{\delta_{\ell-1}}{2}, \epsilon_\ell \leftarrow \frac{3}{4}\epsilon_{\ell-1}$
   \FOR {$g = 1,2, \cdots G$}
    \STATE Initialize pairwise (empirical) win-count $w_{ij} \leftarrow 0$, for each item pair $i,j \in \cG_g$
	\FOR {$\tau = 1, 2, \ldots t\,\,(:= \big\lceil\frac{4k}{m\epsilon_\ell^2}\ln \frac{2k}{\delta_\ell})\big\rceil$}
   	\STATE Play the set $\cG_g$ (one round of battle)
   	\STATE Receive: The top-$m$ ranking $\bsigma_{\tau} \in \bSigma_{\cG}^m$
   	\STATE Update win-count $w_{ij}$ of each item pair $i,j \in \cG_g$ applying \rb\ on $\bsigma_{\tau}$
   	\ENDFOR 
	\STATE Define $\hat p_{i,j} = \frac{w_{ij}}{w_{ij}+w_{ji}}, \, \forall i,j \in \cG_g$
   	\STATE If $\exists$ any $i \in \cG_g$ such that $\hp_{i j}  + \frac{\epsilon_\ell}{2}\ge \frac{1}{2}, \, \forall j \in \cG_g$, then set $c_g \leftarrow i$, else select $c_g \leftarrow$ uniformly at random from $\cG_g$, and set $S \leftarrow S \cup \{c_g\}$
   \ENDFOR
   \STATE $S \leftarrow S \cup \cR_\ell$
   \IF{$(|S| == 1)$}
   \STATE Break (go out of the while loop)
   \ELSIF{$|S|\le k$}
   \STATE $S' \leftarrow $ Randomly sample $k-|S|$ items from $[n] \setminus S$, and $S \leftarrow S \cup S'$, $\epsilon_\ell \leftarrow \frac{c\epsilon}{2}$, $\delta_\ell \leftarrow {\delta}$  
  \ELSE
   \STATE Divide $S$ into $G: = \big \lceil \frac{|S|}{k} \big \rceil$ sets $\cG_1, \cdots \cG_G$ such that $\cup_{j = 1}^{G}\cG_j = S$, $\cG_{j} \cap \cG_{j'} = \emptyset, ~\forall j,j' \in [G], \, |G_j| = k,\, \forall j \in [G-1]$. \textbf{If} $|\cG_{G}| < k$, \textbf{then} set $\cR_{\ell+1} \leftarrow \cG_G$  and $G = G-1$.
   \ENDIF
   \ENDWHILE
   \STATE {\bfseries Output:} The unique item left in $S$
\end{algorithmic}
\end{algorithm}
\vspace{-10pt}
\end{center}

\subsection{Proof of Theorem \ref{thm:alg_ppm}}
\label{app:alg_ppm_thm}

\ubppm*

\begin{proof}
Same as the proof of Thm. \ref{thm:alg_pp}, we start by analyzing the required sample complexity of the algorithm. 
Note that at any iteration $\ell$, any set $\cG_g$ is played for exactly $t= \frac{4k}{m\epsilon_\ell^2}\ln \frac{2k}{\delta_\ell}$ many number of times. Also since the algorithm discards away exactly $k-1$ items from each set $\cG_g$, hence the maximum number of iterations possible is $\lceil  \ln_k n \rceil$. Now at any iteration $\ell$, since $G = \Big\lfloor \frac{|S_\ell|}{k} \Big\rfloor < \frac{|S_\ell|}{k}$, the total sample complexity for iteration $\ell$ is at most $\frac{|S_\ell|}{k}t \le \frac{4n}{mk^{\ell-1}\epsilon_\ell^2}\ln \frac{2k}{\delta_\ell}$, as $|S_\ell| \le \frac{n}{k^\ell}$ for all $\ell \in [\lfloor  \ln_k n \rfloor]$. Also note that for all but last iteration $\ell \in [\lfloor  \ln_k n \rfloor]$, $\epsilon_\ell = \frac{\epsilon}{8}\bigg( \frac{3}{4} \bigg)^{\ell-1}$, and $\delta_\ell = \frac{\delta}{2^{\ell+1}}$.
Moreover for the last iteration $\ell = \lceil  \ln_k n \rceil$, the sample complexity is clearly $t= \frac{4k}{mc^2(\epsilon/2)^2}\ln \frac{4k}{\delta}$, as in this case $\epsilon_\ell = \frac{c\epsilon}{2}$, and $\delta_\ell = \frac{\delta}{2}$, and $|S| = k$.
Thus the total sample complexity of Algorithm \ref{alg:ppm} is given by 

\begin{align*}
\sum_{\ell = 1}^{\lceil  \ln_k n \rceil}\frac{|S_\ell|}{m(\epsilon_\ell/2)^2}& \ln \frac{2k}{\delta_\ell}  
 \le \sum_{\ell = 1}^{\infty}\frac{4n}{mc^2k^\ell\bigg(\frac{\epsilon}{8}\big(\frac{3}{4}\big)^{\ell-1}\bigg)^2}k\ln \frac{k 2^{\ell+1}}{\delta} + \frac{16k}{mc^2\epsilon^2}\ln \frac{4k}{\delta}\\ 
& \le \frac{256n}{mc^2\epsilon^2}\sum_{\ell = 1}^{\infty}\frac{16^{\ell-1}}{(9k)^{\ell-1}}\Big( \ln \frac{k}{\delta} + {(\ell+1)} \Big) + \frac{16k}{mc^2\epsilon^2}\ln \frac{4k}{\delta} \\
& \le \frac{256n}{mc^2\epsilon^2}\ln \frac{k}{\delta}\sum_{\ell = 1}^{\infty}\frac{4^{\ell-1}}{(9k)^{\ell-1}}\Big( {3\ell} \Big) + \frac{16k}{mc^2\epsilon^2}\ln \frac{4k}{\delta} = O\bigg(\frac{n}{mc^2\epsilon^2}\ln \frac{k}{\delta}\bigg) ~[\text{for any } k > 1].\end{align*}

We are now only left with proving the $(\epsilon,\delta)$-{PAC} correctness of the algorithm. We used the same notations as introduced in the proof of Thm. \ref{thm:alg_pp}. 

We start by making a crucial observation that at any phase, for any subgroup $\cG_g$, the strongest item of the $\cG_g$ gets picked in the top-$m$ ranking quite often. More formally:

\begin{lem}
\label{lem:divbat_n1} 
Consider any particular set $\cG_g$ at any phase $\ell$, and let us denote by $q_{i}$ as the number of times any item $i \in \cG_g$ appears in the top-$m$ rankings when items in the set $\cG_g$ are queried for $t$ rounds. Then if $i_g := \arg \max_{i \in \cG_g}\theta_i$, then with probability at least $\Big(1-\frac{\delta_\ell}{2k}\Big)$, one can show that $q_{i_g} > (1-\eta)\frac{mt}{k}$, for any $\eta \in \big(\frac{3}{32\sqrt 2},1 \big]$.
\end{lem} 

\begin{proof}
Fix any iteration $\ell$ and a set $\cG_g$, $g \in 1,2, \ldots, G$. Define $i^\tau_g: = \1(i \in \bsigma_\tau)$ as the indicator variable if $i^{th}$ element appeared in the top-$m$ ranking at iteration $\tau \in [t]$.  Recall the definition of {TR} feedback model (Sec. \ref{sec:feed_mod}). Using this we get $\E[i_g^\tau] = Pr(\{i_g \in \bsigma) = Pr\big( \exists j \in [m] ~|~ \sigma(j) = i_g \big) = \sum_{j = 1}^{m}Pr\big( \sigma(j) = i_g \Big) = \sum_{j = 0}^{m-1}\frac{1}{k-j} \ge \frac{m}{k}$, as $Pr(\{i_g | S\}) \ge \frac{1}{|S|}$ for any $S \subseteq [\cG_g]$ ($i_g := \arg \max_{i \in \cG_g}\theta_i$ being the best item of set $\cG_g$). Hence $\E[q_{i_g}] = \sum_{\tau = 1}^{t}\E[i_g^\tau] \ge \frac{mt}{k}$. 
Now applying Chernoff-Hoeffdings bound for $w_{i_g}$, we get that for any $\eta \in (\frac{3}{32},1]$, 

\begin{align*}
Pr\Big( q_{i_g} \le (1-\eta)\E[q_{i_g}] \Big) & \le \exp(- \frac{\E[q_{i_g}]\eta^2}{2}) \le \exp(- \frac{mt\eta^2}{2k}) \\
& = \exp\bigg(- \frac{2\eta^2}{\epsilon_\ell^2} \ln \bigg( \frac{2k}{\delta_\ell} \bigg) \bigg) = \exp\bigg(- \frac{(\sqrt 2\eta)^2}{\epsilon_\ell^2} \ln \bigg( \frac{2k}{\delta_\ell} \bigg) \bigg) \\
& \le \exp\bigg(- \ln \bigg( \frac{2k}{\delta_\ell} \bigg) \bigg) \le \frac{\delta_\ell}{2k} ,
\end{align*}
where the second last inequality holds as $\eta \ge \frac{3}{32\sqrt 2}$ and $\epsilon_\ell \le \frac{3}{32}$, for any iteration $\ell \in \lceil \ln n \rceil$; in other words for any $\eta \ge \frac{3}{32\sqrt 2}$, we have $\frac{\sqrt 2\eta}{\epsilon_\ell} \ge 1$ which leads to the second last inequality. Thus we finally derive that 
with probability at least $\Big(1-\frac{\delta_\ell}{2k}\Big)$, one can show that $q_{i_g} > (1-\eta)\E[q_{i_g}] \ge (1-\eta)\frac{tm}{k}$, and
the proof follows henceforth.
\end{proof}

In particular, fixing $\eta = \frac{1}{2}$ in Lemma \ref{lem:divbat_n1}, we get that with probability at least $\big(1-\frac{\delta_\ell}{2}\big)$,  $q_{i_g} > (1-\frac{1}{2})\E[w_{i_g}] > \frac{mt}{2k}$. 
Note that, for any round $\tau \in [t]$, whenever an item $i \in \cG_g$ appears in the top-$m$ set $\cG_{gm}^\tau$, then the rank breaking update ensures that every element in the top-$m$ set gets compared with rest of the $k-1$ elements of $\cG_g$. Based on this observation, we now prove that for any set $\cG_g$, a near-best ($\epsilon_\ell$-optimal of $i_g$) is retained as the winner $c_g$ with probability at least $\big( 1 - \frac{\delta_\ell}{2}\big)$. More formally:

\begin{lem}
\label{lem:divbat_n2} 
Consider any particular set $\cG_g$ at any iteration $\ell$. Let $i_g \leftarrow \arg \max_{i \in \cG_g}\theta_i$, then with probability at least $\Big(1-{\delta_\ell}\Big)$, $\theta_{c_g} > \theta_{i_g} - \frac{\epsilon_\ell}{c}$.
\end{lem}

\begin{proof}
With top-$m$ ranking feedback, the crucial observation lies in that at any round $\tau \in [t]$, whenever an item $i \in \cG_g$ appears in the top-$m$ ranking $\bsigma_{\tau}$, then the rank breaking update ensures that every element in the top-$m$ set gets compared to each of the rest $k-1$ elements of $\cG_g$ - it defeats to every element preceding item in $\sigma \in \Sigma_{\cG_{gm}}$, and wins over the rest.  
If $n_{ij} = w_{ij} + w_{ji}$ denotes the number of times item $i$ and $j$ are compared after rank-breaking, for $i,j \in \cG_g$, $n_{ij} = n_{ji}$, and from Lemma \ref{lem:divbat_n1} with $\eta = \frac{1}{2}$ we have that $n_{i_g j} \ge \frac{mt}{2k}$ with probability at least $(1-\delta_\ell/2k)$.
Given the above arguments in place, for any item $j \in \cG_g \sm \{i_g\}$, by Hoeffdings inequality:

\begin{align*}
Pr\big( \big\{ \hp_{j i_g} - p_{j i_g| \cG_g} > \frac{\epsilon_\ell}{2}\big\} \cap \big\{ n_{j i_g} \ge \frac{mt}{2k} \big\} \big) \le \exp\Big(-2\frac{mt}{2k}{(\epsilon_\ell/2)}^2 \Big) \big) \le \frac{\delta_\ell}{2k},
\end{align*}

Now consider any item $j$ such that $\theta_{i_g} - \theta_{j} > \epsilon_\ell/c $, then we have $\frac{Pr(i_g|\cG_g)}{Pr(j|\cG_g)} > 1 + 4 \epsilon_\ell$, which by Lem. \ref{lem:pp_1i} implies $p_{i_g j | \cG_g} > \frac{1}{2} + \epsilon_\ell$, or equivalently $p_{ j i_g| \cG_g} < \frac{1}{2} - \epsilon_\ell$.

But since we show that for any item $j \in \cG_g \sm \{1\}$, with high probability $(1-\delta_\ell/2k)$, we have $\hp_{j i_g} - p_{j i_g|\cG_g} < \frac{\epsilon_\ell}{2}$. Taking union bound above holds true for any $j \in \cG_g \sm \{1\}$ with probability at least $(1-\delta/2)$. Combining with the above claim of $p_{ j i_g| \cG_g} < \frac{1}{2} - \epsilon_\ell$, this further implies $\hp_{j i_g} + \frac{\epsilon_\ell}{2} < p_{j i_g|\cG_g} + {\epsilon_\ell} < \frac{1}{2}$. Thus no such $\epsilon_\ell$ suboptimal item can be picked as $c_g$ for any subgroup $\cG_g$, at any phase $\ell$. 

On the other hand, following the same chain of arguments note that $\hp_{i_g j} - p_{i_g j|\cG_g} > -\frac{\epsilon_\ell}{2} \implies \hp_{i_g j} + \frac{\epsilon_\ell}{2} > p_{i_g j|\cG_g} > \frac{1}{2}$ for all $j \in \cG_g$, $i_g$ is a valid candidate for $c_g$ always, or in other case some other $\epsilon_\ell$-suboptimal item $j$ (such $\theta_{j} > \theta_{i_g} - \epsilon_\ell$) can be chosen as $c_g$. This concludes the proof.
\end{proof}

The correctness-claim now follows using a similar argument as given for the proof of Thm. \ref{thm:alg_pp}. We add the details below for the sake of completeness:
Without loss of generality, we assume the best item of the \rumk\, model is $\theta_1$, i.e. $\theta_1 > \theta_i \, \forall i \in [n] \sm \{1\}$.
Now for any iteration $\ell$, let us define $g_\ell \in [G]$ to be the index of the set that contains \textit{best item} of the entire set $S_\ell$, i.e. $\arg\max_{i \in S_\ell}\theta_i \in \cG_{g_\ell}$. Then applying Lemma \ref{lem:divbat_n2}, with probability at least $(1-\delta_\ell)$,\, $\theta_{c_{g_\ell}} > \theta_{i_{g_\ell}} - \epsilon_\ell/c$. Note that initially, at phase $\ell = 1$, $i_{g_\ell} = 1$. Then, for each iteration $\ell$, applying Lemma \ref{lem:divbat_n2} recursively to $\cG_{g_\ell}$, we finally get $\theta_{r} > \theta_1 -\Big( \frac{\epsilon}{8} + \frac{\epsilon}{8}\Big(\frac{3}{4}\Big) + \cdots + \frac{\epsilon}{8}\big(\frac{3}{4}\big)^{\lfloor \ln_k n \rfloor} \Big) - \frac{\epsilon}{2} \ge \theta_1 -\frac{\epsilon}{8}\Big( \sum_{i = 0}^{\infty}\big( \frac{3}{4} \big)^{i} \Big) - \frac{\epsilon}{2} \ge \theta_1 - \epsilon$.
Thus assuming the algorithm does not fail in any of the iteration $\ell$, we finally have that $p_{r_*1} > \frac{1}{2} - \epsilon$---this shows that the final item output by \algp\, is $\epsilon$ optimal.

Finally note that since at each iteration $\ell$, the algorithm fails with probability at most $\delta_\ell(1/2 + \frac{1}{2k}) \le \delta_\ell$, the total failure probability of the algorithm is at most $\Big( \frac{\delta}{4} + \frac{\delta}{8} + \cdots + \frac{\delta}{2^{\lceil \frac{n}{k}  \rceil}} \Big) + \frac{\delta}{2} \le \delta$.
This shows the correctness of the algorithm, concluding the proof.
\end{proof}

\subsection{Proof of Theorem \ref{thm:lb_pacpl_rank}}
\label{app:alg_ppm_lm}

The proof proceeds almost same as the proof of Thm. \ref{thm:lb_pacpl_win}, the only difference lies in the analysis of the KL-divergence terms with \tf. 

Consider the exact same set of \rumk\, instances, $\{\bnu^a\}_{a = 1}^n$ we constructed for Thm. \ref{thm:lb_pacpl_win}. It is now interesting to note that how the \tf\, affects the KL-divergence analysis, precisely the KL-divergence shoots up by a factor of $m$ which in fact triggers an $\frac{1}{m}$ reduction in regret learning rate. We show this below formally.

Note that for \tf\, for any problem instance $\bnu^a, \, a \in [n]$, each $k$-set $S \subseteq [n]$ is associated to ${k \choose m} (m!)$ number of possible outcomes, each representing one possible ranking of set of $m$ items of $S$, say $S_m$. 
Also the probability of any permutation $\bsigma \in \bSigma_{S}^{m}$ is given by
$
p^a_S(\bsigma) = Pr_{\bnu^a}(\bsigma|S),
$
where $Pr_{\bnu^a}(\bsigma|S)$ is as defined for \tf\, for \rumk\, problem instance $\bnu^a$ (see Sec. \ref{sec:tr}). More formally,
for problem \textbf{Instance-a}, we have that: 
\begin{align*}
p^a_S(\bsigma) & = Pr_{\bnu_a}(\bsigma = \sigma|S) = \prod_{i = 1}^{m}Pr(X_{\sigma(i)} > X_{\sigma(j)}, ~\forall j \in \{i+1, \ldots m\}), \, \forall \bsigma \in \bSigma_S^m\\
& = Pr_{\bnu_a}(\bsigma = \sigma|S) = \prod_{i = 1}^{m}Pr(\zeta_{\sigma(i)} > \zeta_{\sigma(j)} - (\theta_{\sigma(i)} - \theta_{\sigma(j)}^a), ~\forall j \in \{i+1, \ldots m\}), \, \forall \bsigma \in \bSigma_S^m\\
\end{align*}

As also argued in the proof of Thm. \ref{thm:lb_pacpl_win}, note that for any top-$m$ ranking of $\bsigma \in \bSigma_S^m$, $KL(p^1_S(\bsigma), p^a_S(\bsigma)) = 0$ for any set $S \not\owns a$. Hence while comparing the KL-divergence of instances $\bnu^1$ vs $\bnu^a$, we need to focus only on sets containing $a$ (recall we denote this as $S^a$). Applying the chain rule for KL-divergence, we now get:

\begin{align}
\label{eq:lb_witf_kl}
\nonumber KL(p^1_S, p^a_S) = KL(p^1_S(\sigma_1),& p^a_S(\sigma_1)) + KL(p^1_S(\sigma_2 \mid \sigma_1), p^a_S(\sigma_2 \mid \sigma_1)) + \cdots \\ 
& + KL(p^1_S(\sigma_m \mid \sigma(1:m-1)), p^a_S(\sigma_m \mid \sigma(1:m-1))),
\end{align}
where we abbreviate $\sigma(i)$ as $\sigma_i$ and $KL( P(Y \mid X),Q(Y \mid X)): = \sum_{x}Pr\Big( X = x\Big)\big[ KL( P(Y \mid X = x),Q(Y \mid X = x))\big]$ denotes the conditional KL-divergence. 
Moreover it is easy to note that for any $\sigma \in \Sigma_{S}^m$ such that $\sigma(i) = a$, we have $KL(p^1_S(\sigma_{i+1} \mid \sigma(1:i)), p^a_S(\sigma_{i+1} \mid \sigma(1:i))) := 0$, for all $i \in [m]$.

Now using the KL divergence upper bounds, as derived in the proof of Thm. \ref{thm:lb_pacpl_win}, we have than 
\[
KL(p^1_S(\sigma_1), p^a_S(\sigma_1)) \le \frac{ \Delta_a'^2}{\frac{8\epsilon^2}{k}}.
\]

One can potentially use the same line of argument to upper bound the remaining KL divergence terms of \eqref{eq:lb_witf_kl} as well. More formally note that for all $i \in [m-1]$, we can show that:
\begin{align*}
KL&(p^1_S(\sigma_{i+1} \mid \sigma(1:i)), p^a_S(\sigma_{i+1} \mid \sigma(1:i))) \\
& = \sum_{\sigma' \in \Sigma_S^i}Pr(\sigma')KL(p^1_S(\sigma_{i+1} \mid \sigma(1:i))=\sigma', p^a_S(\sigma_{i+1} \mid \sigma(1:i))=\sigma')\le \frac{8\epsilon^2}{k}
\end{align*}

Thus applying above in \eqref{eq:lb_witf_kl} we get:

\begin{align}
\label{eq:lb_witf_kl2}
KL(p^1_S, p^a_S) & = KL(p^1_S(\sigma_1) + \cdots + KL(p^1_S(\sigma_m \mid \sigma(1:m-1)), p^a_S(\sigma_m \mid \sigma(1:m-1))) \le \dfrac{ 8m\epsilon^2}{k}.
\end{align}

Eqn. \eqref{eq:lb_witf_kl2} precisely gives the main result to derive Thm. \ref{thm:lb_pacpl_rank}. Note that it shows an $m$-factor blow up in the KL-divergence terms owning to \tf. The rest of the proof follows exactly the same argument used in \ref{thm:lb_pacpl_win} which can easily be seen to yield the desired sample complexity lower bound.

\end{document}